\documentclass[10pt, conference, letterpaper]{IEEEtran}

\def\thisismainpaper{0}   

\usepackage[dvips]{graphicx}
\usepackage[cmex10]{amsmath}
\usepackage{multirow}
\usepackage{epsfig}
\usepackage{mathrsfs}
\usepackage{amssymb}
\usepackage{comment}
\usepackage{hyperref}
\usepackage{dsfont}
\usepackage{array}
\usepackage{amsthm}
\usepackage{booktabs}
\usepackage{blkarray}
\usepackage{enumerate}
\usepackage{url}
\usepackage{chemarrow}
\usepackage{float}
\usepackage{color}
\usepackage[lined,ruled,linesnumbered,noend]{algorithm2e}
\SetKwRepeat{Do}{do}{while}%
\usepackage{epsf,psfrag}
\usepackage{epsfig}
\usepackage{bbm}
\usepackage{bm}
\def\argmin{\mathop{\arg\min}}

\if\thisismainpaper1
\fi

\theoremstyle{definition}
\newtheorem{theorem}{Theorem}[section]
\newtheorem{lemma}[theorem]{Lemma}

\def\scale#1{{\small #1}}

\newcommand{\E}{{\rm I\kern-.3em E}}

\def\diag{\operatorname{diag}}

\def\dhat{\widehat{\Delta}}
\def\h1{H^{(1)}}

\IEEEoverridecommandlockouts
\title{Time-varying Mixing Matrix Design for Energy-efficient Decentralized Federated Learning}
\author{
\IEEEauthorblockN{Xusheng Zhang$^{*,1}$, Tuan Nguyen$^{*,2}$, and Ting He$^{2}$}
\IEEEauthorblockA{
$^1$ University of Oxford, Oxford, England, UK. Email: xusheng.zhang@cs.ox.ac.uk\\
$^2$ Pennsylvania State University, University Park, PA, USA. Email: \{tmn5319,tinghe\}@psu.edu 
}
\thanks{*Both authors contributed equally to the paper.}
\thanks{
Part of this work was presented at IEEE INFOCOM 2026~\cite{Zhang26INFOCOM}. 
}
}

\begin{document}
%
\maketitle
\begin{abstract}
We consider the design of mixing matrices to minimize the operation cost for decentralized federated learning (DFL) in wireless networks, with focus on minimizing the maximum per-node energy consumption. As a critical hyperparameter for DFL, the mixing matrix controls both the convergence rate and the needs of agent-to-agent communications, and has thus been studied extensively. However, existing designs mostly focused on minimizing the communication time, leaving open the minimization of per-node energy consumption that is critical for energy-constrained devices. This work addresses this gap through a theoretically-justified solution for mixing matrix design that aims at minimizing the maximum per-node energy consumption until convergence, while taking into account the broadcast nature of wireless communications. Based on a novel convergence theorem that allows arbitrarily time-varying mixing matrices, we propose a multi-phase  design framework that activates time-varying communication topologies under optimized budgets to trade off the per-iteration energy consumption and the convergence rate while balancing the energy consumption across nodes. Our evaluations based on real data have validated the efficacy of the proposed solution in combining the low energy consumption of sparse mixing matrices and the fast convergence of dense mixing matrices.   
\end{abstract}

\IEEEpeerreviewmaketitle

\section{Introduction}

\emph{Decentralized federated learning (DFL)}~\cite{Lian17NIPS} is an emerging machine learning paradigm that allows distributed learning agents to collaboratively learn a shared model from the union of their local data without directly sharing the data. Instead of coordinating through a parameter server as in centralized {federated learning (FL)}~\cite{McMahan17AISTATS}, the agents participating in DFL directly exchange model updates with their neighbors through peer-to-peer communications, which are then aggregated locally~\cite{Kairouz21book}. DFL has attracted significant attention since its introduction, because compared to its centralized counterpart, DFL has better robustness by avoiding a single point of failure and better balances the communication loads across nodes without increasing the computational complexity~\cite{Lian17NIPS}.

Meanwhile, DFL still faces significant performance challenges due to the extensive data transfer between learning agents. 
Such challenges are particularly prominent for deep learning due to the large model size, where communication cost can dominate the total cost of the learning task~\cite{Luo20MLsys}, particularly in resource-constrained edge networks~\cite{chen2022federated}. 
This issue has attracted tremendous research interests in reducing the communication cost of DFL, such as methods for reducing the amount of data per communication through compression (e.g., \cite{Compression1}) and methods for reducing the number of communications through hyperparameter optimization (e.g., \cite{Chiu23JSAC}) or adaptive  communications (e.g., \cite{Singh20CDC}). 

In particular, the design of \emph{mixing matrix} used for local parameter aggregation plays a crucial role, as each non-zero off-diagonal entry in the mixing matrix will trigger an agent-to-agent communication. 
In this regard, most existing works focused on accelerating learning by minimizing the communication time, either measured in the maximum number of neighbors an agent communicates with~\cite{MATCHA22,Chiu23JSAC,hua2022efficient,le2023refined} or the number of time slots for scheduling all the communications~\cite{Herrera25OJCS}. However, in wireless networks such as HetNets \cite{chen2022federated}, device-to-device networks \cite{xing2021federated}, and IoT networks~\cite{pinyoanuntapong2022toward}, an often more important cost measure is energy consumption, which has received less attention. Only a few works have tried to optimize energy consumption during DFL~\cite{Chiu23JSAC,zhang2024energyefficient}, but they either did not consider the balance of energy consumption between nodes~\cite{Chiu23JSAC} or did not utilize the broadcast nature of wireless communications~\cite{zhang2024energyefficient}. 

In this work, we aim at filling this gap with the objective of \emph{minimizing the maximum energy consumption per node} for DFL to reach a given level of convergence,  
while taking into account the broadcast nature of wireless communications. Our solution is built upon a novel convergence theorem under time-varying mixing matrices, based on which we propose a multi-phase mixing matrix design framework together with concrete algorithms that designs randomized mixing matrices under optimized per-node budgets to optimally balance the convergence rate and the energy consumption at each node.

\subsection{Related Works}

\textbf{Decentralized federated learning.} First explored by \cite{Lian17NIPS} through Decentralized Parallel Stochastic Gradient Descent (D-PSGD), DFL removes the central server in \cite{McMahan17AISTATS} and enables model training over peer-to-peer networks. A central research question is how DFL performs compared to centralized FL, particularly in convergence rate, communication cost, and generalization. Subsequent studies  including \cite{D2ICML18,ICMLhonor,Xin21} advance DFL in both algorithms and theories, though the main focus remains on reducing the number of iterations needed for convergence. \looseness=-1

\textbf{Communication cost reduction.} 
As model sizes continue to grow, communication overhead has become a key bottleneck limiting the performance and reliability of DFL on wireless edge networks \cite{LiYangGSWCL21}. Existing work for reducing communication cost can be broadly categorized into three approaches. The first is to lower the cost per communication round using compression techniques \cite{Compression1,Compression2,Compression3,Zhang20INFOCOM}. The second is to reduce the total number of communication rounds 
\cite{sysml19,Ngu19INFOCOM,Wang19JSAC,Wang21JMLR}. A third line of work focuses on activating only selected subsets of links rather than all links simultaneously. In this direction, event-triggered mechanisms were introduced in \cite{Singh20CDC,Singh21JSAIT}, and optimized (possibly randomized) communication patterns were proposed in \cite{MATCHA22,Chiu23JSAC,Herrera25OJCS}.
  
In this work, we aim at designing the communication patterns as in \cite{MATCHA22,Chiu23JSAC,Herrera25OJCS}, which has the advantage of providing predictable performance compared to event-triggered mechanisms, but \emph{we consider a different objective}. 
While most existing works on communication design focused on minimizing the communication time, measured by the number of matchings \cite{MATCHA22,Chiu23JSAC}, the maximum degree \cite{hua2022efficient,le2023refined}, or the number of collision-free transmission slots \cite{Herrera25OJCS}, we focus on minimizing the maximum energy consumption per node, which is critical for wireless networks \cite{chen2022federated,xing2021federated,pinyoanuntapong2022toward}. Compared to the communication time, fewer works have tackled the optimization of energy consumption~\cite{Chiu23JSAC,zhang2024energyefficient,DeVos24IPDPSW}, and the existing solutions either ignored the balance of energy consumption across nodes~\cite{Chiu23JSAC} or ignored the broadcast nature of wireless communications~\cite{zhang2024energyefficient}. Some solution~\cite{DeVos24IPDPSW} even ignored the energy consumption by communications. 
\emph{This work aims at filling this gap by designing the communication patterns to minimize the maximum energy consumption per node, while considering broadcast communications.} 
While a few works considered broadcast-based DFL \cite{Chen23SPAWC, Herrera25OJCS}, they had different optimization objectives (e.g., maximizing \#successful links~\cite{Chen23SPAWC} or minimizing \#transmission slots~\cite{Herrera25OJCS}).

\textbf{Mixing matrix design in DFL.}
The mixing matrix, i.e., the matrix containing the local aggregation weights, is an important hyperparameter in DFL. The impact of the mixing matrix on the convergence rate of DFL is usually captured through its spectral gap~\cite{Lian17NIPS,neglia2020decentralized,jiang2023joint} or an equivalent parameter that denotes the discrepancy between the designed mixing matrix and the ideal mixing matrix under all-to-all communications~\cite{MATCHA22}. 
Although recent studies have pointed out additional  parameters that can affect convergence, such as the effective number of neighbors~\cite{vogels2022beyond} and the neighborhood heterogeneity~\cite{le2023refined}, these results do not invalidate the importance of the spectral gap.    
Based on the identified convergence parameters, several mixing matrix designs have been proposed to balance the convergence rate and the cost per iteration~\cite{MATCHA22,Chiu23JSAC,Herrera25OJCS,hua2022efficient,le2023refined}. 
In this regard, our mixing matrix design is also based on a parameter related to the spectral gap, but \emph{we build on a generalized convergence theorem allowing time-varying mixing matrices} (including random matrices with time-varying distributions), which enables a multi-phase design. 

\textbf{DFL over time-varying topology.}
Most gossip-based distributed optimization algorithms have been analyzed under fixed network topologies, with only a few works establishing convergence guarantees for time-varying topologies.
\cite{LiuChenZheng22} studies D-PSGD with multiple steps of local training and multiple steps of gossip, which can be viewed as periodically alternating between the full base topology and an empty topology.
\cite{Koloskova20ICML} provides the state-of-the-art convergence analysis for D-PSGD with very general (randomized) mixing matrices, subject only to a spectral condition in each period.
Algorithm subgradient-push \cite{NedicOlshevsky15} also handles time-varying topologies, but their convergence analyses rely on convexity of the loss functions and additional spectral assumptions on the mixing matrices.
Gradient tracking algorithms such as DIGing \cite{NedicOlshevskyShi17} and Acc-GT \cite{LiLin24} have likewise been studied over time-varying topologies  under convex loss functions.
The convergence analysis in \cite{HuangYuan22} is the most general to date, applicable to a broad class of algorithms, requiring neither convexity nor strong assumptions on the topology class; however, it still assumes a uniform lower bound on the spectral gap of the time-varying topologies, an assumption not made in this paper.
\emph{The convergence theorem presented here focuses on D-PSGD and applies to topologies with arbitrarily time-varying spectral gaps, covering both convex and non-convex objectives.}

\textbf{Adaptive communications.} Adapting communications during training, e.g., via adapting the communicated content (e.g., adaptive model compression~\cite{Honig22, Zhuansun24arXiv}, adaptive model pruning~\cite{WangXuPei25}, adaptive self-distillation~\cite{He22, LiWLDHD25}, adaptive parameter freezing~\cite{Chen21ICDCS}) or the act of communication (e.g., adapting communication period~\cite{Tchaye-Kondi24TMC}, adapting client selection~\cite{ZhangXRWC21}, adapting topology construction~\cite{Xu22}, and adapting both communication period and topology~\cite{LiaoXuXu24}), has been shown to improve the performance of FL.
Among these studies, only \cite{Xu22, LiaoXuXu24} address decentralized federated learning (DFL), the setting we study here. Their adaptive mechanisms, however, come with non-trivial computational overhead or system complexity: \cite{Xu22} integrates real-time deep reinforcement learning at every iteration, while \cite{LiaoXuXu24} relies on a central coordinator to compute complicated adaptive decisions. Moreover, the resulting adaptations depend on real-time decisions, and thus the behavior would be difficult to predict or analyze.
Our approach moves in a different direction. Based on a generalized convergence theorem, we can \emph{analytically characterize the impact of fine-grained adaptations of the act of communication}, without requiring real-time algorithmic decision-making (from a centralized orchestration), and our method remains compatible with content-level adaptations for further performance gains.

\subsection{Summary of Contributions}

We consider the mixing matrix design for broadcast-based DFL, with the objective of minimizing the maximum per-node energy consumption, with the following contributions: 

1) Motivated by our initial experiments that suggest the benefit of varying the communication intensity during training, we derive a new convergence theorem that allows \emph{arbitrarily time-varying} (random) mixing matrices, which generalizes existing convergence theorems that require mixing matrices to be fixed~\cite{Lian17NIPS}, i.i.d.~\cite{MATCHA22}, $B$-connected~\cite{NedicOlshevsky15}, or periodic~\cite{Koloskova20ICML, LiuChenZheng22}.

2) Based on the theorem, we propose a multi-phase design framework as well as a corresponding trilevel optimization algorithm, and derive explicit expressions of its objective function in the cases of 1-phase and 2-phase design. 

3) We develop a budgeted mixing matrix design algorithm to solve the lower-level optimization under broadcast communications, and analytically characterize the convergence rate under its design in the case of fully connected base topology.  

4) We evaluate the proposed solution against baselines and state-of-the-art benchmarks under realistic settings for learning in wireless networks. Our results validate the efficacy of the proposed solution in combining the low energy consumption of sparse mixing matrices and the fast convergence of dense mixing matrices to improve the energy efficiency of DFL. 

\textbf{Roadmap.} Section~\ref{sec:Background and Formulation} states the problem formulation, Section~\ref{sec:Convergence Analysis} presents our convergence theorem, based on which Sections~\ref{sec:Optimization Framework}--\ref{sec:Lower-level optimization for broadcast communication} develop the proposed solution, Section~\ref{sec:Performance Evaluation} presents the performance evaluation, and Section~\ref{sec:Conclusion} concludes the paper.  
\textbf{All proofs are deferred to the appendix}.

\section{Background and Problem Formulation}\label{sec:Background and Formulation}

\subsection{Decentralized Learning Problem}
Consider a network of $m$ nodes linked via a base topology $G = (V, E)$, where $V$ is the set of nodes ($|V| = m$) and $E$ indicates which node pairs can exchange information directly. 
 That is, each node $i$ can only directly communicate with nodes in its one-hop neighborhood (including itself), denoted by $V_i :=\{i\}\cup \{j\in V: (i,j)\in E\}$.
Each node $i \in V$ is associated with a local objective function $F_i(\bm{x})$, which is a function of the parameter $\bm{x} \in \mathds{R}^d$ that depends on the local dataset at node $i$. The objective of DFL is to collaboratively minimize the global function
\begin{align}\label{eq:F(x)}
F(\boldsymbol{x}) := \frac{1}{m}\sum_{i=1}^{m} F_{i}(\boldsymbol{x}),
\end{align}
which averages the local objectives across all nodes.
We assume that $F$ attains its minimum value $F_{\inf}$.

\subsection{Decentralized Learning Algorithm}\label{subsec:Decentralized Learning Algorithm}

We focus on a well-established decentralized learning algorithm known as Decentralized Parallel Stochastic Gradient Descent (D-PSGD)~\cite{Lian17NIPS}. Let $\bm{x}_i^{(t)}$ represent the model parameter at node $i$ at the start of iteration $t$, and let $g(\bm{x}_i^{(t)}; \xi_i^{(t)})$ denote the stochastic gradient computed at that node using a minibatch $\xi_i^{(t)}$ drawn from its local data and the current parameter $\bm{x}_i^{(t)}$.
At each iteration, D-PSGD performs the following update in parallel at every node $i$:
\looseness=-1
\scale{
\begin{align}\label{eq:DecenSGD}
    \boldsymbol{x}^{(t+1)}_i = \sum_{j=1}^{m}\bm{W}^{(t)}[i,j](\boldsymbol{x}^{(t)}_j - \eta g(\boldsymbol{x}^{(t)}_j; \xi^{(t)}_j)),
\end{align}
}where $\bm{W}^{(t)}=(\bm{W}^{(t)}[i,j])_{i,j=1}^m$ is the $m \times m$ mixing matrix used at iteration $t$, and $\eta > 0$ denotes the learning rate. 

Since node $i$ requires communication from node $j$ in iteration $t$ only when $\bm{W}^{(t)}[i,j] \neq 0$, the communication pattern can be controlled by appropriately designing the mixing matrix $\bm{W}^{(t)}$.
According to \cite{Lian17NIPS}, the mixing matrix should be \emph{topology-compliant} ($\bm{W}^{(t)}[i,j]\neq 0$ only if $(i,j)\in E$) and \emph{symmetric with each row/column summing up to one}\footnote{In \cite{Lian17NIPS}, the mixing matrix was assumed to be symmetric and \emph{doubly stochastic} with entries constrained to $[0,1]$, but we find that all the convergence proofs only required the mixing matrix to be symmetric with each row/column summing up to one.} in order to ensure feasibility and convergence for D-PSGD. 

In general, the mixing matrix $\bm{W}^{(t)}$ can be a random matrix drawn from a distribution $\bm{\mathcal{W}}^{(t)}$ as long as all its realizations satisfy our desired properties, i.e., being topology-compliant and symmetric with  row/column sums of $1$. In this case, by $\E[\phi(\bm{W}^{(t)})]$ we denote the expectation of a function $\phi$ of $\bm{W}^{(t)}$, 
where $\bm{W}^{(t)}$ is drawn from its distribution $\bm{\mathcal{W}}^{(t)}$.

Let $\overline{\bm{x}}^{(t)}:={1\over m}\sum_{i=1}^m \bm{x}^{(t)}_i$ denote the learned global model at iteration $t$ (which is a global average of the local models at this iteration). Depending on the shape of the objective function $F$, there are two widely adopted convergence criteria: for any \emph{required level of convergence} $\epsilon>0$, we say that D-PSGD has achieved $\epsilon$-convergence if \begin{itemize}
    \item $\frac{1}{T} \sum_{t=0}^{T-1} (\E[F(\overline{\bm{x}}^{(t)})] - F_{\inf} )\le \epsilon$ when the local objective functions $F_i$ are convex, or 
    \item $ \frac{1}{T} \sum_{t=0}^{T-1} \E[\|\nabla F(\boldsymbol{\overline{\bm{x}}}^{(t)})\|^2]\leq \epsilon$ 
    for general (possibly non-convex) $F_i$\footnote{In this work, we use $\|\bm{a}\|$ to denote $\ell$-2 norm if $\bm{a}$ is a vector, and spectral norm if $\bm{a}$ is a matrix.}.
\end{itemize}

The convergence of D-PSGD can be guaranteed under the following assumptions:
\begin{enumerate}[(1)]
    \item Each local objective function $F_i(\boldsymbol{\bm{x}})$ is $L$-Lipschitz smooth, i.e., 
    for all $\bm{x},\bm{x}'\in  \mathds{R}^d$, $\|\nabla F_i(\bm{x})-\nabla F_i(\bm{x}')\|\leq L\|\bm{x}-\bm{x}'\|,\: \forall i\in V$.
\item There exist constants $M_1, \hat{\sigma}$ such that 
    \scale{
    $\hspace{-0em} {1\over m}\sum_{i\in V}\E[\|g(\bm{x}_i;\xi_i)-\nabla F_i(\bm{x}_i) \|^2] \leq \hat{\sigma}^2 + {M_1\over m}\sum_{i\in V}\|\nabla F(\bm{x}_i)\|^2$, 
    $\forall \bm{x}_1,\ldots,\bm{x}_m \in \mathds{R}^d$.
    }
    \item There exist constants $M_2, \hat{\zeta}$ such that \scale{
    ${1\over m}\sum_{i\in V}\|\nabla F_i(\bm{x})\|^2\leq \hat{\zeta}^2 + M_2\|\nabla F(\bm{x}) \|^2, \forall \bm{x} \in \mathds{R}^d$.
    }
\end{enumerate}
While the convergence of D-PSGD has also been proved under other assumptions \cite{Lian17NIPS,MATCHA22}, the above assumptions, originally from \cite{Koloskova20ICML}, are more general; see \cite{Koloskova20ICML} for explanations.

\subsection{Cost Models}

We focus on energy consumption as the cost measure. Specifically, we use $c_i(\bm{W}^{(t)})$ to denote the energy consumption at node $i$ in an iteration when the mixing matrix is $\bm{W}^{(t)}$, which contains two parts: (i) computation energy $c_i^a$ for computing the local stochastic gradient and the local aggregation, and (ii) communication energy for sending the parameter vector at node $i$ to each of its activated neighbors. 
Under \emph{broadcast communications} where node $i$ can send its parameter vector to all its neighbors through a broadcast transmission, 
we model the per-iteration energy consumption at node $i$ as
\begin{align}\label{eq:cost definition - broadcast}
    c_i(\bm{W}^{(t)}) := c^a_i + c^b_i \mathbbm{1}(\exists (i,j)\in E:\: \bm{W}^{(t)}[i,j]\neq 0),
\end{align}
where $c^b_i$ is the energy consumption per transmission, and $\mathbbm{1}(\cdot)$ denotes the indicator function. We say that node $i$ is \emph{activated} in iteration $t$ if $\mathbbm{1}(\exists (i,j)\in E:\: \bm{W}^{(t)}[i,j]\neq 0)=1$. 
The total energy consumption at node $i$ over $T$ iterations is thus 
$\sum_{t=1}^T c_i(\bm{W}^{(t)})$. 

An alternative cost model under \emph{unicast communications} has been studied in \cite{Chiu23JSAC, zhang2024energyefficient},  where node $i$ needs to separately transmit its parameter vector over each activated link $(i,j)$ with energy consumption $c_{ij}^b$. In this setting, the energy consumption at node $i$ in iteration $t$ becomes 
\begin{equation}    
\label{eq:cost def - unicast}
c_i(\bm{W}^{(t)}) := c^a_i + \sum_{j: (i,j)\in E} c^b_{ij} \mathbbm{1}(\bm{W}^{(t)}[i,j]\neq 0). 
\end{equation}

\textbf{Remark: }The above models ignore random factors such as link dynamics and retransmissions, which are left to future work. 

\subsection{Design Objective}
\label{subsec:Design Objective}
There is an inherent tradeoff between  converging in fewer iterations and spending less energy per iteration, which can be controlled by designing possibly time-varying mixing matrices $(\bm{W}^{(t)})_{t=1}^T$. To maximize the lifetime of the learning system, we aim at designing the mixing matrices to \emph{minimize the maximum per-node energy consumption}, defined as
\begin{align}\label{eq:max per-node energy}
\max_{i\in V} \sum_{t=1}^T c_i(\bm{W}^{(t)}),
\end{align}
until the learning task reaches a required level of convergence. \looseness=0
\looseness=0


\subsection{Motivating Experiment}\label{subsec:Motivating Experiment}

\begin{figure}[t!]
    \centering    \centerline{\mbox{\includegraphics[height = 2.35in,width=1.0\linewidth]{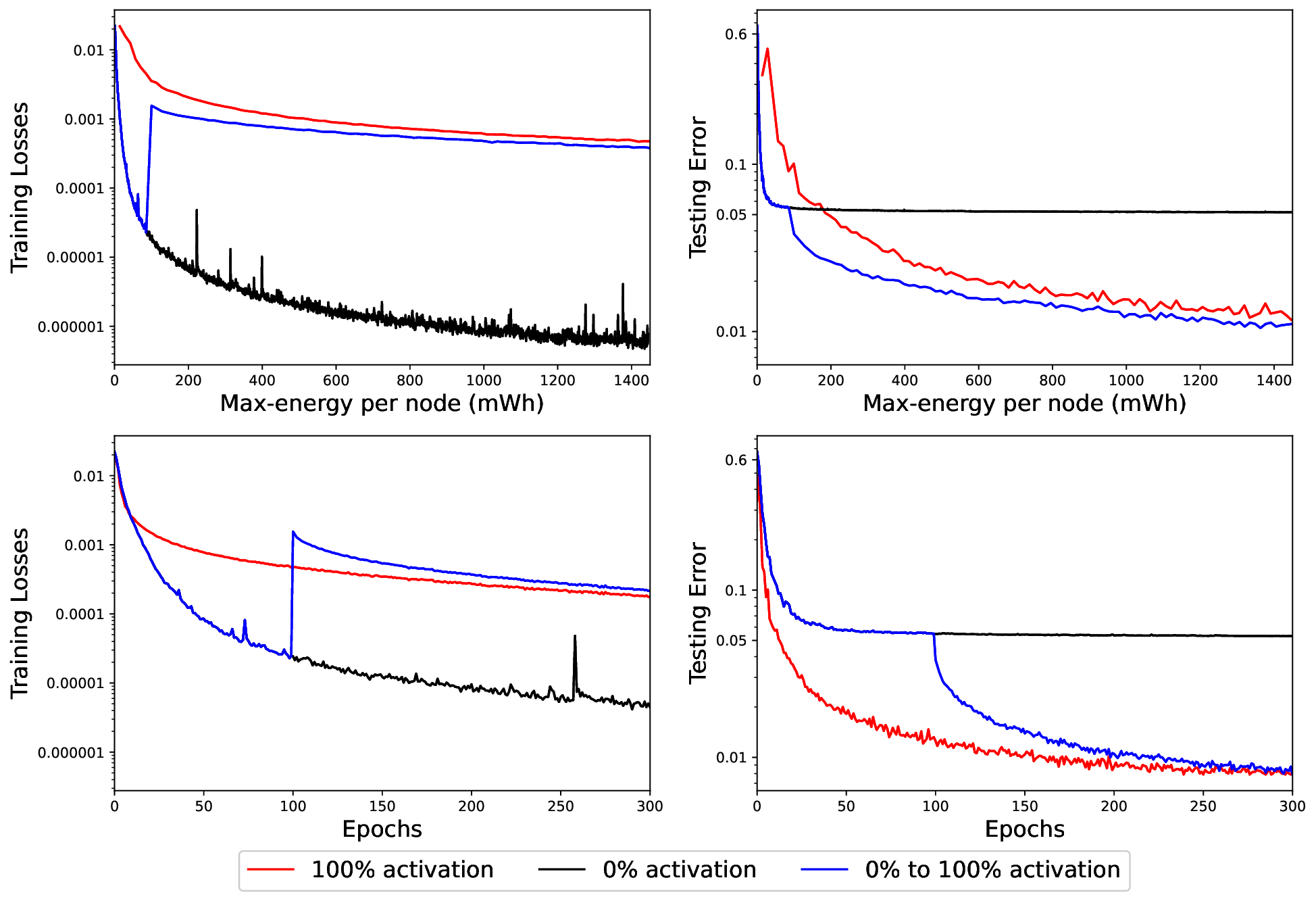}}}
    \vspace{-1em}
     \caption{ Motivating experiment based on MNIST over a 100-node clique (x-axis truncated to show the same range for all curves). 
    }
    \label{fig:results_mnist_clique}
    \vspace{-1em}
\end{figure}


As a motivating example, we compare the tradeoff between learning performance and energy consumption under three example designs. Our experiment is based on the MNIST dataset, randomly distributed across 100 nodes fully connected with each other, each training a 4-layer CNN model with 1,663,370 parameters as in \cite{McMahan17AISTATS}, with a batch size of 64 and a learning rate of $0.05$. We set the energy consumption parameters as in Section~\ref{sec:Performance Evaluation}. 
We evaluate the training loss and the testing error under three designs: 
\begin{enumerate}[(1)]
    \item $\bm{W}^{(t)}\equiv {1\over m}\bm{1}\bm{1}^\top$, i.e., all the nodes broadcast their local models in all the iterations (`$100\%$ activation'),
    \item $\bm{W}^{(t)}\equiv \bm{I}$, i.e., no communication between nodes (`$0\%$ activation'), and 
    \item a 2-phase design that starts with $\bm{I}$ and switches to ${1\over m}\bm{1}\bm{1}^\top$ after 100 epochs (`$0\%$ to $100\%$ activation').  
\end{enumerate}
The trajectories showing training progress as energy is spent, as shown in Fig.~\ref{fig:results_mnist_clique}, suggest that while using a dense mixing matrix yields a lower testing error under high energy budgets, using a sparse mixing matrix is more efficient under low energy budgets due to the savings in communication energy.
The low training loss of `$0\%$ activation' is due to overfitting, which also explains the initial increase in training loss for the 2-phase design after switching.
Suitably switching between the two can achieve a better tradeoff between the quality of learning and the energy consumption than each case alone. 


\section{Convergence Analysis}\label{sec:Convergence Analysis}

The foundation of our solution is a set of convergence theorems that characterize the number of iterations until convergence as explicit functions of the mixing matrices. To present these theorems, we will need a few notions: the \emph{ideal mixing matrix} $\bm{J}:={1\over m}\bm{1} \bm{1}^\top$ which maintains global consensus, the \emph{divergence} between a designed (possibly random) mixing matrix $\bm{W}$ and the ideal mixing matrix defined as $$\rho(\bm{W}):= \|\E[\bm{W}^\top\bm{W}]-\bm{J}\|,$$ and the following  parameters that capture the dependence of convergence rate on time-varying mixing matrices.


We define
\[
p^{(t)} := 1 - \rho\big(\bm{W}^{(t)}\big),
\]
and let $\bm{p} := \{p^{(t)}\}_{t=0}^{\infty}$.
For each $j \ge 0$, define
\[
\pi_j \;:=\; \sum_{i>j} \prod_{t=j+1}^{i-1}
\Big(1-\frac{p^{(t)}}{2}\Big).
\]

The quantity $\pi_j$ aggregates the contraction effects of time-varying
mixing matrices $\{\bm{W}^{(t)}\}_{t \ge j+1}$ starting from iteration $j+1$.
Each factor $1-p^{(t)}/2$ upper bounds the amount of divergence at time $t$,
while the product reflects mixing over consecutive iterations
accumulated over time. Consequently, $\pi_j$ provides an upper bound on the cumulative divergence from
the averaging operator $J$ that persists after the $j$-th iteration.

For a time horizon $T>0$, we further define the time-averaged quantities
\begin{align}
\Pi_1(T) &:=\frac{1}{T} \sum_{j=0}^{T-1} \pi_j, \quad
\Pi_2(T) :=\frac{1}{T} \sum_{j=0}^{T-1} \frac{\pi_j}{p^{(j)}}. \label{eq:Pi_1} 
\end{align}
The quantity $\Pi_1(T)$ represents the average cumulative divergence over the
first $T$ iterations, while $\Pi_2(T)$ is a weighted version that additionally
accounts for the instantaneous mixing strength $p^{(j)}$ at the starting time
$j$. Both can be viewed as ergodic measures of how effectively the sequence of
time-varying mixing matrices promotes consensus over time.

We collect a few useful facts about the parameters  $\pi_j$, $\Pi_1(T)$, and $\Pi_2(T)$.
\begin{lemma}
\label{fact: pi}
Assume there exists \(\delta \in (0,1)\) such that \(p^{(t)} \ge \delta\) for all \(t \ge 0\). Then the following statements hold:
    \begin{enumerate}
        \item For all $j\ge 0$, $\pi_j < \frac{2}{\delta}$.
        \item For any $T > 0 $, $\Pi_1(T) < \frac{2}{\delta}$ and $\Pi_2(T)< \frac{2}{\delta^2}$.
        \item In particular, if $p^{(t)}\equiv p$, then for any $T>0$, 
        $\Pi_1(T) = \pi_0= 2/p$ and $\Pi_2(T) = 2/p^2$.
    \end{enumerate}
\end{lemma}

The following theorem describes how the convergence rate depends on the mixing matrices through the parameters $\Pi_1(T)$, $\Pi_2(T)$, $\pi_0$, and $p_{\min}:= \min_{j} p^{(j)}$. 

\begin{theorem}\label{thm: new convergence bound nonconvex}
D-PSGD under assumptions (1)--(3) in Section~\ref{subsec:Decentralized Learning Algorithm}  achieves $\epsilon$-convergence (i.e., $\frac{1}{T} \sum_{t=0}^{T-1} \E[ \|\nabla F(\overline{\bm{x}}^{(t)}) \|^2] \le \epsilon$) when the number of iterations $T$ satisfies 
    $T\geq T_1(\Pi_1(T), \Pi_2(T), \pi_0, p_{\min}, \epsilon,  \overline{\bm{x}}^{(0)})$ 
    for 
    \begin{align}
        &\hspace{-.0em}T_1(\Pi_1, \Pi_2, \pi_0, p_{\min}, \epsilon,  \overline{\bm{x}}^{(0)}) := O\hspace{-.25em}\left(\hspace{-.25em}  \frac{f_0 L\sqrt{(1+M_1) (1+M_2)}}{\epsilon p_{\min}}\right)
        \nonumber \\
    &\hspace{1em} +O\left(\frac{f_0 L^2 [(1+\pi_0)\Xi_0 + (1+M_1)]}{\epsilon} 
    \right) \nonumber\\
     &\hspace{1em} +f_0 L \cdot O\left(
    {\hat{\sigma}^2\over m\epsilon^2} +{\sqrt{(M_1 \hat{\zeta}^2+\hat{\sigma}^2)\Pi_1 + \hat{\zeta}^2\Pi_2 } \over  \epsilon^{3/2}} \label{eq: T1 bound}
    \right),
    \end{align}
    where $f_0:=\E[ F(\overline{\bm{x}}^{(0)})] - F_{\inf}$ ($F_{\inf}$ denotes the minimum value of $F$) and
    $\Xi_0:=\frac{1}{m}\sum_{i=1}^m \|x_i^{(0)} - \overline{\bm{x}}^{(0)} \|^2$.
\end{theorem}
The big-$O$ notations used in the Theorem above and throughout the manuscript hide absolute constants independent of all the parameters.

\textbf{Remark~1:}
From Theorem~\ref{thm: new convergence bound nonconvex}, the required number of iterations depends on $\bm{p}$ only through $\Pi_1(T)$, $\Pi_2(T)$, $\pi_0$, and $p_{\min}$. 
When $p^{(t)}\equiv p$, we have $\Pi_1(T)\equiv 2/p$ and $\Pi_2(T) \equiv 2/p^2$, in which case Theorem~\ref{thm: new convergence bound nonconvex} reduces to the state-of-the-art convergence theorem in \cite{Koloskova20ICML} in the case of $\tau=1$. 
Furthermore, with a modified definition of $p^{(t)}$ to characterize $\tau$ matrices in consecutive iterations, the proof of Theorem~\ref{thm: new convergence bound nonconvex} can be adapted to support the convergence conditions in \cite{Koloskova20ICML} for  general $\tau\ge 1$. 

\textbf{Remark~2:} 
We briefly discuss the contributions of each term in \eqref{eq: T1 bound}. 
The \textbf{third term} is expected to be the dominant one: while the term 
$O(\hat{\sigma}^2/(m\epsilon^2))$ is asymptotically tight for any stochastic methods \cite{NemirovskyYudin85, ICMLhonor}, the 
$O(\epsilon^{-3/2})$ part highlights the main novelty in our bound, as it only 
involves the ``average'' parameters $\Pi_1$ and $\Pi_2$ rather than worst-case parameters. 
The \textbf{second term} scales as $O(\epsilon^{-1})$ and depends on $\pi_0$. 
This dependence is fairly innocuous, especially since one may force $\Xi_0=0$ by fixing the same initialization point on all worker nodes; 
we include such dependnece only for completeness. 
Finally, the \textbf{first term} depends on the worst-case parameter 
$p_{\min}$ in the sequence $\bm{p}$. While this dependence may appear 
suboptimal, the term itself scales only as $O(\sqrt{M_1 M_2}/\epsilon)$ and is 
independent of $\hat{\zeta}$ and $\hat{\sigma}$, so its overall impact is 
not expected to be significant.

\textbf{Remark~3:}
Analysis of decentralized gradient-based optimization algorithms with time-varying mixing matrices has been 
studied in \cite{NedicOlshevsky15, NedicOlshevskyShi17}, but their convergence 
guarantees rely on convexity of $F_i$'s and other strong assumptions on both the mixing 
matrices and the objective functions $\{F_i\}_{i=1}^m$. Our techniques also extend to the convex 
setting:  we establish a convergence theorem analogous to Theorem~\ref{thm: new convergence bound nonconvex} showing that the 
algorithm achieves $\epsilon$-convergence, i.e., 
$\tfrac{1}{T} \sum_{t=0}^{T-1} (\E [F(\overline{\bm{x}}^{(t)})] - F_{\inf}) 
\le \epsilon$, under an additional assumption that the local objectives are convex. Moreover, in this case, our requirements (1)--(3) can be further relaxed; in particular, we may assume $M_1=M_2=0$. 
See Appendix~\ref{sec: convex convergence} for details.

The convergence condition in Theorem~\ref{thm: new convergence bound nonconvex} may be satisfied by multiple choices of $T$. We thus define $T_2(\bm{p}, \epsilon, \overline{\bm{x}}^{(0)})$ as the smallest $T$ that satisfies the convergence condition for a given sequence $\bm{p}$,  error bound $\epsilon$, and initial model $\overline{\bm{x}}^{(0)}$, i.e.,  
\begin{align}
&T_2(\bm{p}, \epsilon, \overline{\bm{x}}^{(0)}) := \nonumber\\
&\hspace{0em}\min \{T>0: 
        T\ge T_1(\Pi_1(T), \Pi_2(T), \pi_0, p_{\min}, \epsilon, \overline{\bm{x}}^{(0)})
    \}.
\end{align}
If the local objective functions are convex, we define $T_2$ analogously as 
\begin{align}
    &T_2(\bm{p}, \epsilon, \overline{\bm{x}}^{(0)}) := \nonumber\\
&\hspace{0em}\min \{T>0: 
        T\ge T_4(\Pi_1(T), \Pi_2(T), \pi_0, p_{\min}, \epsilon, \overline{\bm{x}}^{(0)})
\end{align}where $T_4$ denotes the convergence bound given in Theorem~\ref{thm: new convergence bound}.

\section{Optimization Framework}\label{sec:Optimization Framework}

At a high level, our design objective as stated in Section~\ref{subsec:Design Objective} can be formulated as the following optimization problem:
\begin{subequations}\label{eq:adaptive mixing matrix design}
\begin{align}
& \min_{ T, \{\bm{W^{(t)}}\}_{t=1}^T }\:  \max_{i\in V} \sum_{t=1}^T c_i(\bm{W}^{(t)})  \label{deterministic design:obj}\\
\mbox{s.t. }
& T \ge  T_2(\bm{p}, \epsilon, \overline{\bm{x}}^{(0)}), 
\label{eq:nonconvex constraint} \\
& p^{(t)} = 1-\rho(\bm{W}^{(t)}),~~~\forall t=1,\ldots,T.
\end{align}
\end{subequations}
The design variables include both the number of iterations $T$ and the sequence of mixing matrices$\{\bm{W}^{(t)}\}_{t=1}^T$.  
We treat \(\epsilon\) and the initial model \(\bm{x}^{(0)}\) as explicit inputs to the formulation.


\subsection{Randomized Multi-phase Design Framework}
We note that under broadcast communications, deterministic mixing matrix design is not sufficient. 
This is because if any iteration $t$ uses a deterministic mixing matrix $\bm{W}^{(t)}$ for which the set of communicating nodes $U$ is not equal to $V$, then $\bm{W}^{(t)}$ is a reducible mixing matrix and thus $\rho(\bm{W}^{(t)})= \|\bm{W}^{(t)}-J \|^2=1$, which implies $\Pi_2(T)=\infty$ and thus the right-hand side of \eqref{eq:nonconvex constraint} is infinite. 
Therefore, any deterministic mixing matrix design must activate communications at all the nodes in all the iterations.  
 Under this constraint, the mixing matrix design problem reduces to the simple problem of designing a single mixing matrix $\bm{W}$ with the minimum $\rho(\bm{W})$ based on the entire base topology $G$, whose solution can be rather suboptimal in energy efficiency as shown in Section~\ref{subsec:Motivating Experiment} (`$100\%$ activation'). 
By employing a randomized mixing matrix design, we can avoid this constraint and significantly enlarge the design space. 
Therefore, we generalize the problem from a deterministic design problem \eqref{eq:adaptive mixing matrix design} to a randomized design problem, 
by replacing its objective \eqref{deterministic design:obj} with
$$
 \min_{ T, \{\bm{W^{(t)}}\}_{t=1}^T }\: 
    \E\left[
    \max_{i\in V} \sum_{t=1}^T 
    c_i(\bm{W}^{(t)}) 
    \right].
$$

To simplify the computation, we divide the training process into $K$ \emph{phases}, where $K$ is a design variable. Each phase $s\in \{1,\dots K\}$ contains $\tau_s$ consecutive iterations of using the random mixing matrix $\bm{W}_s$. 
For clarity, we 
use $\bm{W}_s^{(t)}$ to refer to the mixing matrix used in the $t$-th iteration of phase $s$, which is an instance of $\bm{W}_s$ drawn independently in each iteration.
This changes our design problem into the following optimization: 
\begin{subequations}\label{eq:adaptive mixing matrix design, phased}
\begin{align}
&    \min_{K, \{\bm{W_{s}}, \tau_s\}_{s=1}^K }\:  
    \E\left[
    \max_{i\in V} \sum_{s=1}^K 
    \sum_{t=1}^{\tau_s}
    c_i(\bm{W}_{s}^{(t)}) 
    \right] \label{eq:obj over K phases} \\
\mbox{s.t. }
& \tau_1 + \dots + \tau_K = T_2(\bm{p}, \epsilon, \overline{\bm{x}}^{(0)}) \label{eq: nonconvex constraint over K phases} \\
& p^{(t)} =  1-\rho(\bm{W}_s),\: \forall~ t=\big(\sum_{l=1}^{{s-1}}\tau_{l},\ldots,\sum_{l=1}^{{s}} \tau_{l} \big], \label{eq: phase requirement}\\
& \hspace{13em} \forall s=1,\ldots,K.\nonumber
\end{align}
\end{subequations}
The decision variables are: the number of phases $K$, the $K$ random mixing matrices $\{\bm{W}_s\}_{s=1}^K$, and the number of iterations $\{\tau_s\}_{s=1}^K$ for using each. 
For a given solution to \eqref{eq:adaptive mixing matrix design, phased}, 
the number of iterations required for convergence is denoted as \looseness=-1
\begin{align}
T_3((p_1, \tau_1), (p_2, \tau_2), \dots, p_K) :=  T_2(\bm{p}, \epsilon, {\bm{x}}^{(0)}),
\end{align}
where $p_s:=1-\rho(\bm{W}_s)$ is the convergence parameter for each phase $s=1,\dots, K$. 
Note that, given $(p_1, \tau_1), (p_2, \tau_2)  \dots, p_K$ and $K$,  $\tau_K$ is no longer an independent variable due to \eqref{eq: nonconvex constraint over K phases}. 

\subsection{A Related Sub-problem}

Since $T_3((p_1, \tau_1), \dots, p_K)$ is a {decreasing} function of each $p_s=1-\rho(\bm{W}_s)$, we want to design a random mixing matrix for each phase $s$ to minimize $\rho(\bm{W}_s)$ without triggering too many communications. To this end, we introduce a sub-problem of budgeted mixing matrix design.  
Let $D>0$ denote the \emph{budget} for the per-iteration expected energy consumption at each node in a given phase. We formulate the per-phase mixing matrix design problem  (omitting the phase index) as 
\begin{subequations}\label{eq:budgeted random W design}
\begin{align}
    & \min_{\Pr[\cdot]} \left\| \sum_{\bm{W}\in \mathcal{M}} \Pr[\bm{W}] \bm{W}^\top \bm{W} - J \right\| \label{eq: broadcast K=1 obj} \\
    & \sum_{\bm{W} \in \mathcal{M}} \Pr[\bm{W}] = 1, \text{ and } \Pr[\bm{W}] \ge 0 ~~ \forall~ \bm{W} \in \mathcal{M},  \\
    & \sum_{\bm{W} \in \mathcal{M}} \Pr[\bm{W}] c_i(\bm{W})  \le D, \forall~ i\in V \label{eq: exp edge constraint}\\
    &  \mathbbm{1}[\Pr[\bm{W}]\cdot \bm{W}[i,j] \neq 0] \le \mathbbm{1}[E[i,j]\neq 0], \label{eq:ramdomized graph constraint} \\
    &\hspace{13em} \forall~ \bm{W}\in \mathcal{M}, \forall~ i\neq j \in V, \nonumber
\end{align}
\end{subequations}
where we use $\mathcal{M}$ to denote the set of symmetric matrices whose row/column sums are $1$, and $\Pr[\bm{W}]$ to denote the probability of choosing $\bm{W}\in \mathcal{M}$ in each iteration of this phase.\looseness=-1 

We will dive into a concrete solution to \eqref{eq:budgeted random W design} in Section~\ref{sec:Lower-level optimization for broadcast communication}-\ref{sec:Lower-level optimization under unicast communication}. To present the overall solution, we assume the existence of a hypothetical solution $\mathcal{A}$ to \eqref{eq:budgeted random W design}, treated as a black box for the time being, with the following property: for any budget $D$, $\mathcal{A}$ provides a random mixing matrix $\bm{W}$ that is feasible for \eqref{eq:budgeted random W design} with the guarantee that there exists a function $\rho_{\mathcal{A}}^+:\mathds{R}^+\rightarrow\mathds{R}^+ \cup \{\infty\}$ such that $\rho(\bm{W})\le \rho^+_{\mathcal{A}}(D)$.

Under a given budget $D$, one can bound the expected maximum energy consumption per node for a phase of $T$ iterations by the following lemma. 
\begin{lemma}\label{lem:bound under budget D for K=1}
Let $D>0$ and $T>0$. Suppose a random mixing matrix $\bm{W}$ satisfies  $\E[c_i(\bm{W})] \le D$ ($\forall i\in V$).
If $\bm{W}^{(1)}, \dots, \bm{W}^{(T)}$ are i.i.d. copies of $\bm{W}$, 
then 
\begin{equation}\label{eq:exp maximum bound}
    \E\left[  \max_i \sum_{t=1}^T c_i(\bm{W}^{(t)}) \right] \le D\cdot \left(T
    + m\sqrt{\frac{T \pi}{8}} \right) 
    =: q(T,D). 
\end{equation}
\end{lemma}
The proof of Lemma~\ref{lem:bound under budget D for K=1} is deferred to Section~\ref{sec: deferred proof}.

\subsection{Case Studies for $K=1$ and $K=2$}

\subsubsection{Computation of $T_3$}
Next we illustrate computation for $T_3$ through case studies of $K=1$ and $K=2$.

In the case of $K=1$, as explained in the first remark below Theorem~\ref{thm: new convergence bound nonconvex}, we have $\Pi_1(T)\equiv 2/p_1$ and $\Pi_2(T)\equiv 2/p_1^2$.
With this simplification,
\begin{align}
T_3(p_1)&= T_1(2/p_1, 2/p_1^2, 2/p_1, p_1, \epsilon,  \overline{\bm{x}}^{(0)}) \nonumber\\
&= O\left(  \frac{f_0 L\sqrt{(1+M_1) (1+M_2)}}{\epsilon p_1}\right) \nonumber\\
&    +O\left(\frac{f_0 L^2 [\Xi_0 + p_1(1+M_1)]}{\epsilon p_1} 
    \right) \nonumber\\
&     +f_0 L \cdot O\left(
    {\hat{\sigma}^2\over m\epsilon^2} +{\sqrt{(M_1 \hat{\zeta}^2+\hat{\sigma}^2) p_1 + \hat{\zeta}^2  } \over  \epsilon^{3/2}p_1}
    \right).
\end{align}
In the case of $K=2$, $\Pi_1(T)$ and $\Pi_2(T)$ are given by the following lemma.
\begin{lemma}\label{lem:T((p1,tau1),p2)}
Let $T\ge \tau_1$ be a fixed integer.
If $p^{(1)}=\cdots=p^{(\tau_1)}=p_1$ and $p^{(\tau_1+1)}=\cdots=p^{(T)}=p_2$, then $\Pi_1(T)$ and $\Pi_2(T)$ as defined in \eqref{eq:Pi_1} satisfy the followings:
\begin{align*}
    \Pi_1(T) &= \frac{2(T-\tau_1)}{T p_2} + \frac{2\tau_1}{T p_1} -  (\frac{2}{T p_1} - \frac{2}{Tp_2}) \sum_{j=1}^{\tau_1} (1-\frac{p_1}{2})^{j}, \\ 
    \Pi_2(T)  &= \frac{2(T-\tau_1)}{T p_2^2} + \frac{2\tau_1}{T p_1^2} -  (\frac{2}{T p_1^2} - \frac{2}{Tp_2 p_1}) \sum_{j=1}^{\tau_1} (1-\frac{p_1}{2})^{j}.
\end{align*}
\end{lemma}
The proof of Lemma~\ref{lem:T((p1,tau1),p2)} is elementary and will be provided in Appendix~\ref{sec: deferred proof}.
From Lemma~\ref{lem:T((p1,tau1),p2)} we have $\Pi_1(T)\rightarrow 2/p_2$ and $\Pi_2(T) \rightarrow 2/p_2^2$ as $T$ grows. 
Therefore, by a linear search over $T=\tau_1, \tau_1 +1, \dots$, the first $T$ satisfying $T\ge T_1(\Pi_1(T), \Pi_2(T), \pi_0, p_{\min}, \epsilon, \overline{\bm{x}}^{(0)})$ is the value of $T_3((p_1, \tau_1), p_2)$. \looseness=-1

\subsubsection{Computation of Design Objective}

Lemma~\ref{lem:bound under budget D for K=1} 
together with the function $\rho^+_{\mathcal{A}}(D)$ enables us to upper-bound the objective function \eqref{eq:obj over K phases}.
As an example, consider $K=1$.  
By Lemma~\ref{lem:bound under budget D for K=1}, the random mixing matrix $\bm{W}$ designed by $\mathcal{A}$ for budget $D$ will achieve convergence with an expected maximum energy consumption per node no more than  
\begin{equation}
\label{eq: QK=1}
Q_{K=1}(D) := 
q\big(T_3(1-\rho_{\mathcal{A}}^+(D)), D\big).
\end{equation}
Thus, by relaxing the objective function \eqref{eq:obj over K phases} into its upper bound \eqref{eq: QK=1},  we can obtain a 1-phase (randomized) mixing matrix design by minimizing $Q_{K=1}(D)$ over $D \in \mathds{R}^+$ and then feeding the resulting $D$ into the given subroutine $\mathcal{A}$ to obtain a random mixing matrix $\bm{W}$. 

\begin{algorithm}[tb]
\small
\SetKwInOut{Input}{input}\SetKwInOut{Output}{output}
\Input{Maximum \#phases $\overline{K}$, subroutine $\mathcal{A}$ for the problem \eqref{eq:budgeted random W design} and an associated function $\rho^+_{\mathcal{A}}$.}
\Output{Mixing matrices $\bigcup_{s=1}^{K^*}\{\bm{W}_s^{(t)}\}_{t=1}^{\tau^*_s}$.}
\For{$K \gets 1$ \KwTo $\overline{K}$}{
Minimize $Q_K=\sum_{s=1}^{K} q(\tau_s, D_s)$ for $\tau_K = T_3((p_1,\tau_1), \dots, (p_{K-1}, \tau_{K-1}), p_K) - \sum_{s=1}^{K-1} \tau_s$ and $p_s=1-\rho^+_{\mathcal{A}}(D_s)$\; \label{MMMD:2}
}
$K^* \gets \argmin_K Q_K$, with the corresponding solution $D^*_1, \dots, D^*_{K^*}, \tau^*_1, \dots, \tau^*_{K^*-1}$\;
\For{$s \gets 1$ \KwTo $K^*$}{
    \For{$t \gets 1$ \KwTo $\tau^*_s$}{
           $\bm{W}_s^{(t)} \gets \mathcal{A}(D^*_s)$\;
    }
}
return $\bigcup_{s=1}^{K^*}\{\bm{W}_s^{(t)}\}_{t=1}^{\tau^*_s}$\;
\caption{Multi-phase Mixing Matrix Design}
\label{alg:framework}
\vspace{-.05em}
\end{algorithm}

Now consider the case of $K=2$. 
This enlarges our design space to include distinct positive budgets $D_1>0$ and $D_2>0$ for each phase, as well as the number of iterations $\tau_1$ for phase $1$. 
We first rewrite the objective function \eqref{eq:obj over K phases} specific to $K=2$:
\begin{align}
&\min_{ \{\bm{W_{s}}, \tau_s\}_{s=1}^K }\:  
    \E\left[
    \max_{i\in V} \sum_{s=1}^K 
    \sum_{t=1}^{\tau_s}
    c_i(\bm{W}_{s}^{(t)}) 
    \right] =
    \nonumber\\
& \min_{ \bm{W}_1, \bm{W}_2, \tau_1}\hspace{-.5em} \E\left[ 
    \max_{i\in V} \big[ \sum_{t=1}^{\tau_1}c_i(\bm{W}_1^{(t)}) + \hspace{-.5em} \sum_{t=\tau_1+1}^{T_3( (p_1, \tau_1), p_2)}\hspace{-.5em} c_i(\bm{W}_2^{(t-\tau_1)}) \big]
\right] \label{eq:K=2 simplified objective randomized}
\end{align}
Given a solution $\mathcal{A}$ to the sub-problem \eqref{eq:budgeted random W design}, we apply $\mathcal{A}$ to $D_1$ and $D_2$, respectively. 
Then, based on the performance bound $\rho^+_{\mathcal{A}}$, the objective \eqref{eq:K=2 simplified objective randomized} of 2-phase design can be upper-bounded by
\begin{equation}
    \label{eq:Q K=2}
   Q_{K=2}(D_1, D_2, \tau_1) := 
   q(\tau_1, D_1) + q(T_3((p_1, \tau_1), p_2) - \tau_1, D_2), 
\end{equation}
where $p_1 = 1- \rho^+_{\mathcal{A}}(D_1)$ and $p_2 = 1- \rho_{\mathcal{A}}^+(D_2)$. 
Therefore, optimizing $Q_{K=2}(D_1, D_2, \tau_1)$ over the $3$-tuple $(D_1, D_2, \tau_1)$ will yield an optimized 2-phase design. 

\subsection{Overall Solution}

We propose to solve the multi-phase design problem~\eqref{eq:adaptive mixing matrix design, phased} through a trilevel optimization as shown in Algorithm~\ref{alg:framework}:
\begin{itemize}
    \item \noindent \textbf{Upper-level optimization:}~Decide the number of phases $K$ to minimize the overall objective.

    \item \noindent \textbf{Intermediate-level optimization:}~Given a number of phases $K$, minimize the relaxed objective $Q_K:=\sum_{s=1}^{K} q(\tau_s, D_s)$ to determine the budget $D_s$ and the duration $\tau_s$ of each phase $s=1,\dots, K$.

    \item \noindent \textbf{Lower-level optimization:}~Given the budget $D_s$ of each phase, design a random instance $\bm{W}_s^{(t)}$ of the mixing matrix for each iteration in this phase by applying the given solution to the sub-problem \eqref{eq:budgeted random W design}.
\end{itemize}
The focus of Algorithm~\ref{alg:framework} is on the intermediate-level optimization, which optimizes a function of $2K-1$ variables (i.e., $D_1, \dots, D_K, \tau_1, \dots, \tau_{K-1}$). Given the performance bound $\rho_{\mathcal{A}}^+$ of the lower-level optimization, this intermediate-level optimization only has an $O(K)$-dimensional solution space, which is much more tractable than directly optimizing $K$ $m\times m$ random matrices.  

\textbf{Remark:} In the remainder of this work, we will solve the upper/intermediate-level optimizations numerically for small values of $K$ to focus on solving the lower-level optimization under broadcast communications. 
Our trilevel optimization framework is generally applicable under any communication model. 
We leave the development of more efficient solutions to the upper/intermediate-level optimizations to future work.

\section{Lower-level optimization under broadcast communication}\label{sec:Lower-level optimization for broadcast communication}

In this section, we focus on developing efficient solutions to the budgeted problem defined in \eqref{eq:budgeted random W design} under the broadcast communication cost model.

\subsection{Algorithm Design}

\begin{algorithm}[t]
\caption{Budgeted Mixing Matrix Design for Broadcast Communication}
\label{alg:randomized-W for broadcast}
\small
\KwIn{A base topology $G=(V,E)$ with per-computation  cost $c_i^a$ and per-transmission cost $c_i^b$ ($\forall i\in V$) and per-node  budget $D$. }
\KwOut{A mixing matrix \(\bm{W}\) with an expected per-node cost no more than $D$.}

\BlankLine
\textbf{Step 1: Sample a subset of nodes to activate} \\
\Indp
Let $U \gets \emptyset$ denote the set of activated nodes\;
\ForEach{\(i \in V\)}{
    Add node \(i\) to \(U\) independently with probability
    $\min((D-c_i^a)/c_i^b,1)$\;
}
\Indm

\BlankLine
\textbf{Step 2: Design entries for nodes not in \(U\)} \\
\Indp
\ForEach{\(i \in V \setminus U\)}{
   \(\bm{W}[i,i] \gets 1\)\;
   \(\bm{W}[i,j]\gets 0,\: \bm{W}[j,i] \gets 0\) for all \( j \neq i\)\;
}
\Indm

\BlankLine
\textbf{Step 3: Design entries for nodes in \(U\)} \\
\Indp
\ForEach{\(i \in U\)}{
    \ForEach{$j\in U \setminus V_i$}
    {
        $\bm{W}[i,j] \gets 0$\; 
    }

    \ForEach{$j\in (V_i \cap U)\setminus \{i\}$}
    {\(\bm{W}[i,j] \gets 1 / \max(|V_i\cap U|, |V_j\cap U|)\)\;}
       
    \(\bm{W}[i,i] \gets 1-\sum_{j\in V\setminus \{i\}}\bm{W}[i,j] \)\;
}
\end{algorithm}

To pose a well-defined and non-trivial problem, we  assume
\begin{align}\label{eq:condition on D - clique}
\max_i c_i^a \le D < \max_i (c_i^a + c_i^b). 
\end{align}
Indeed, if $D\ge \max_i (c_i^a + c_i^b)$, then we can deterministically activate all the nodes, in which case the optimal mixing matrix can be efficiently computed by solving a semi-definite programming problem~\cite{Chiu23JSAC}. 
Meanwhile, $D< \max_i c_i^a$ is not feasible under the cost model in \eqref{eq:cost definition - broadcast} as some node will exceed the budget even if it does not communicate. 

Under the cost model \eqref{eq:cost definition - broadcast}, the problem \eqref{eq:budgeted random W design} becomes 
\begin{subequations}\label{eq:budgeted random W design - clique}
\begin{align}
    & \min_{\Pr[\cdot]} \left\| \sum_{\bm{W}\in \mathcal{M}} \Pr[\bm{W}] \bm{W}^\top \bm{W} - J \right\| \label{eq: broadcast K=1 obj clique} \\
    & \sum_{\bm{W} \in \mathcal{M}} \Pr[\bm{W}] = 1, \text{ and } \Pr[\bm{W}] \ge 0 ~~ \forall \bm{W} \in \mathcal{M}, \label{eq:distribution constraint} \\
    &  c_i^a + 
    \Pr[\exists j\neq i: \bm{W}[i,j]\neq 0] \cdot c_i^b \le D, \quad \forall i\in V \label{eq: broadcast K=1 constraint clique hetero} \\
    &  \Pr\left[\mathbbm{1}[\bm{W}[i,j] \neq 0] \le \mathbbm{1}[E[i,j]\neq 0], \forall i\neq j \in V \right] = 1. \label{eq: edge constraint w.p. 1}
\end{align}
\end{subequations}

We now present Algorithm~\ref{alg:randomized-W for broadcast}, a randomized algorithm designed to solve \eqref{eq:budgeted random W design - clique} in 3 steps.
\begin{enumerate}
    \item In Step~1, we activate a subset of nodes \(U \subseteq V\) by independently selecting each node with a probability reflecting its residual budget after subtracting the computation cost. 
    \item In Step~2, nodes in $V \setminus U$ are assigned weights so that they essentially perform only local updates. 
    \item In Step~3, nodes in $U$ are assigned mixing weights between themselves according to the Metropolis–Hastings rule~\cite{Xiao2006DistributedAC}.
\end{enumerate}
Random client selection—as used in Step 1 of Algorithm \ref{alg:randomized-W for broadcast}—has been a standard technique in centralized FL to improve communication efficiency~\cite{McMahan17AISTATS, WangLinChen22}.
In the decentralized setting, SwarmSGD~\cite{Nadiradze21} can be viewed as a special case of our random client selection in which only two randomly chosen adjacent nodes  gossip. The method in \cite{ZhouLLOWY21} adopts a similar strategy of keeping one random neighbor for each node, but further refines it by optimizing the sampling distribution to favor high-speed edges.
Our random-selection, in contrast, places no restrictions on the number of participating nodes or the number of neighbors each node may have, yet it guarantees that every realization yields a mixing-matrix design tailored to our specific budgeted problem. In particular,  
in Lemma~\ref{lem|:feasibility} we show that the mixing matrix generated by Algorithm~\ref{alg:randomized-W for broadcast} is always topology-compliant and a feasible solution to problem \eqref{eq:budgeted random W design - clique}, e.g., a valid randomized mixing matrix that satisfies the budget $D$ at every node.

\begin{lemma}\label{lem|:feasibility}
    Let $G$ be an arbitrary base topology. 
    Assume \eqref{eq:condition on D - clique} holds for the given $(c_i^a, c_i^b)_{i\in V}$ and $D$.
    The mixing matrix $\bm{W}$ outputted by Algorithm~\ref{alg:randomized-W for broadcast} is a feasible solution to problem \eqref{eq:budgeted random W design - clique}. 
\end{lemma}

\subsection{Performance Analysis}\label{sec:solution M}

\subsubsection{Result for Fully-connected Base Topology}
Consider the case where the base topology is a clique, i.e., every node can receive the broadcast of every other node\footnote{We assume that proper transmission scheduling is in place to avoid collision. The specific schedule is irrelevant under the cost model \eqref{eq:cost definition - broadcast}.}. 
In this case, we can characterize the performance of Algorithm~\ref{alg:randomized-W for broadcast} analytically.

To present this result, we introduce some additional notations. We use $\sim_{m}$ to relate two quantities $A$ and $B$ if $\lim_{m\rightarrow\infty}\frac{A}{B} = 1$, that is, $A$ is \emph{asymptotically equivalent} to $B$.
Also, for each $i\in V$, we let
$$
\omega_i:= \min\left(\frac{D-c_i^a}{c_i^b}, 1\right),$$
and we use $\bm{\omega}$ to denote the vector $[\omega_1, \omega_2, \dots, \omega_m]$. 
Lastly, given an arbitrary vector $\bm{a}$, let $\diag(\bm{a})$ denote  
the diagonal matrix with $\bm{a}$ on the main diagonal, 
and $\bm{a}^2$ denote the vector $[a_1^2, \dots, a_m^2]$. 
Based on these notations, we characterize the performance of Algorithm~\ref{alg:randomized-W for broadcast} as follows. 

\begin{theorem}\label{thm: guarantee heterogeneous}
Assume $G$ is a clique, and  \eqref{eq:condition on D - clique} holds for the given cost vectors $(c_i^a, c_i^b)_{i\in V}$ and budget parameter $D$.
 The mixing matrix $\bm{W}$ designed by Algorithm~\ref{alg:randomized-W for broadcast} satisfies    
    \begin{equation}\label{eq:rho(W) - clique}
    \rho(\bm{W})    \sim_{m} \left\| m^\perp \bm{\omega}\bm{\omega}^\top + \diag(\bm{1}-\bm{\omega} ) - \bm{J} \right\|,
    \end{equation}
    where $m^\perp$ denote the scalar $\E[\frac{1}{|U|} \mid U\neq \emptyset]$ with the expectation taken over the random generation of $U$.
\end{theorem}

Moreover, in the special case of homogeneous cost parameters, i.e., $c_i^a \equiv c^a$, $c_i^b \equiv c^b$, 
we can explicitly express the dependency of $\rho(\bm{W})$ on the cost parameters $(c^a, c^b)$ and the budget $D$. 
Specifically, for any $D$ satisfying $c^a \le D <c^a + c^b$, we have \looseness=-1
\[
\omega_i =  \frac{D-c^a}{c^b} =: \Tilde{\omega},~~~ \forall i \in V.
\]
Given the fact that $m^\perp \sim_m 1/(\Tilde{\omega}m)$, we obtain that
\begin{align}
    \rho(\bm{W} ) 
    &\sim_{m} \left\| \frac{1}{\Tilde{\omega}m} \cdot \Tilde{\omega}^2 \bm{1}\bm{1}^\top + 
    (1 - \Tilde{\omega}) \bm{I}
    - \bm{J}
    \right\| \nonumber\\
    &=   \left\|  (1 - \Tilde{\omega}) (\bm{I} - \bm{J}) \right\| 
     =  1- \Tilde{\omega}  = 1-{D-c^a\over c^b}. \label{eq:rho(W) - clique, homogeneous}
\end{align}

\subsubsection{Discussion on General Base Topology}\label{subsubsec:Discussion on General Base Topology}
Using a similar analysis as Theorem~\ref{thm: guarantee heterogeneous}, one could derive an analogous bound on \(\rho(\bm{W})\) for the $\bm{W}$ designed by Algorithm~\ref{alg:randomized-W for broadcast}. However, the bound will involve $O(m^2)$ conditional expectations analogous to $m^\perp$, which is challenging to compute numerically. In this case, we evaluate $\rho(\bm{W})$ numerically through Monte Carlo experiments, i.e., for a given budget $D$, we generate a large number of matrices $\bm{W}_1, \dots, \bm{W}_N$ from independent runs of Algorithm~\ref{alg:randomized-W for broadcast} and use the empirical mean
$$\hat{\rho}(\bm{W})= \left\|
    \frac{1}{N} \sum_{n=1}^N \bm{W}_n^\top\bm{W}_n-\bm{J}
\right\|$$ to approximate $\rho(\bm{W})$.


\section{Lower-level optimization under unicast communication}\label{sec:Lower-level optimization under unicast communication}

For the budgeted problem in \eqref{eq:budgeted random W design} to be feasible and non-trivial under the unicast cost model \eqref{eq:cost def - unicast}, we assume
\begin{align}\label{eq:condition on D - unicast}
\max_i c_i^a \le D < \max_i \left(c_i^a + \sum_{j:(i,j)\in E} c_{ij}^b\right). 
\end{align}
This section presents a (meta-)algorithm for solving \eqref{eq:budgeted random W design} for any budget $D$ that satisfies the condition in \eqref{eq:condition on D - unicast}, on a general base topology $G=(V,E)$ with arbitrary cost vectors $(c_i^a, c_{ij}^b)_{(i,j)\in E}$.

\subsection{Meta-algorithm Design}

The algorithm starts by employing a given set of (possibly random) graph oracles to generate subgraphs of the base topology $G$ as candidate topologies, with which we then compute candidate mixing matrices and an optimal distribution over the candidates using semi-definite programming (SDP). 
See Algorithm~\ref{alg:randomized-W for unicast new} for its pseudo-code. 



\begin{algorithm}[ht]
\caption{Budgeted Mixing Matrix Design for Unicast Communication}
\label{alg:randomized-W for unicast new}

\KwIn{A base topology $G=(V,E)$ with per-computation  cost $c_i^a$ and per-transmission cost $c_{ij}^b$ ($\forall i,j\in V$) and per-node  budget $D$; 
a set of graph oracles $\{\mathcal{G}^{(k)}_G\}_{k=1}^{K_{\max}}$.}
\KwOut{A list of mixing matrices with their probabilities $\{(\bm{W}_s,p_s)\}_{s=1}^{K_{\max}}$.}

\For{$k=1,2,\ldots,K_{\max} $}{ 
Draw $G_k=(V_k,E_k)$ from $\mathcal{G}_G^{(k)}$ with incidence matrix $\bm{B}_k$\;
Solve \eqref{eq:sdp-weight-design} for $\bm{\alpha}_k$\;
Set \(\bm{W}_k=I - \bm{B}_k\operatorname{diag}(\bm{\alpha}_k)\bm{B}_k^\top\)\;
}
Given candidate matrices $(\bm{W}_1,\dots,\bm{W}_{K_{\max}})$, solve \eqref{eq:budgeted probabilities design} for the optimal probability vector \label{BMM Unicast:10}
\((p_0^\star,p_1^\star,\dots,p_{K_{\max}}^\star)\)\;

\Return $\{(\bm{W}_s,p_s^\star)\}_{s=1}^{K_{\max}}$\;
\end{algorithm}

It can be verified that the output of Algorithm~\ref{alg:randomized-W for unicast new} is a feasible solution to the problem defined in \eqref{eq:budgeted random W design}. Specifically, according to prior works \cite{Chiu23JSAC,zhang2024energyefficient}, the optimal weights for a given activated communication graph $G_k=(V_k,E_k)$ can be computed from the following SDP:\footnote{Here $\bm{B}_k$ is a $|V_k|\times|E_k|$ incidence matrix of $G_k$, defined as $\bm{B}_k[i,j]=1$ if link $e_j$ starts at node $i$ (under arbitrary orientation), $-1$ if $e_j$ ends at $i$, and $0$ otherwise.} 
\begin{equation}
\label{eq:sdp-weight-design}
\begin{aligned}
&\quad \min_{\bm{\alpha}_k}\quad  \rho_k\\
\text{s.t.}\quad
& -\rho_k I
\preceq
I - \bm{B}_k\operatorname{diag}(\bm{\alpha}_k)\bm{B}_k^\top - J
\preceq
\rho_k I.
\end{aligned}
\end{equation}
 At Line~\ref{BMM Unicast:10}, we solve the optimization problem given in \eqref{eq:budgeted probabilities design}. A straightforward calculation shows that the objective \eqref{eq: unicast K=1 obj clique} is equivalent to \eqref{eq: broadcast K=1 obj}. Moreover, the budget constraint in \eqref{eq: unicast K=1 constraint hetero} is a direct instantiation of \eqref{eq: exp edge constraint} under the unicast cost model. Finally, the random graph oracle $\mathcal{G}_G^{(k)}$ ensures that each sampled graph $G_k$ is a subgraph of $G$, thereby automatically satisfying the constraint \eqref{eq:ramdomized graph constraint}. 
\begin{subequations}\label{eq:budgeted probabilities design}
\begin{align}
    & \min_{p_1, \dots, p_{K_{\max}}} \left\| \sum_{k=1}^{K_{\max}}p_k\bm{W}_k^\top \bm{W}_k - J \right\| \label{eq: unicast K=1 obj clique} \\
    & \sum_{k=1}^{K_{\max}} p_k=1 \quad \text{and} \quad p_k\ge 0   ~~\forall ~k=0,\dots, K_{\max},\\
    &  c_i^a + 
    \sum_{k=1}^{K_{\max}} p_k \sum_{j:(i,j)\in E_k} c_{ij}^b 
    \le D, \quad \forall i\in V. \label{eq: unicast K=1 constraint hetero} 
\end{align}
\end{subequations}
Problem \eqref{eq:budgeted probabilities design} is similar to problem (19) in \cite{Chiu23JSAC}, with the key difference that constraint \eqref{eq: unicast K=1 constraint hetero} imposes a ``per-node cost constraint'', whereas the latter enforces a total cost constraint.
 After having a set of $K_{max}$ candidate mixing matrices, we solve  \eqref{eq:budgeted probabilities design}, which an SDP.  
Therefore, the total running time of Algorithm~\ref{alg:randomized-W for unicast new} is polynomial in $m$ and $K_{\max}$.

Algorithm~\ref{alg:randomized-W for unicast new} provides a unified framework that encompasses the approaches proposed in several prior works. Each of these algorithms can be implemented using Algorithm~\ref{alg:randomized-W for unicast new}, equipped with different graph oracles and suitably modified optimization objectives:
\begin{itemize}
\item \cite{MATCHA22} 
introduces Matching Decomposition Sampling, which optimizes the distribution over sets of matchings.
\item \cite{Chiu23JSAC} proposes Laplacian Matrix Sampling, which computes an optimal distribution over a set of (spectral) sparsifiers;
\item \cite{Herrera25OJCS} presents BASS, which designs a distribution over collision-free subgraphs; 
\item \cite{zhang2024energyefficient} proposes a Ramanujan topology for fully connected base topology, and a greedy heuristic for general base topology by iteratively solving \eqref{eq:sdp-weight-design} and removing the link with the minimum absolute weight. 
\end{itemize}
Below, we present a layered graph construction oracle with performance guarantee.  


Let $c^b_{i}:= \max_{j: (i,j)\in E} c^b_{ij}$. The transmission cost at node $i$ is bounded by \(c_i^a+c_i^b\sum_{j:(i,j)\in E}\mathbbm{1}\{\bm W[i,j]\neq 0\}\). We can enforce budget $D$ by bounding the degree of each node by 
\[d_i:= \left\lfloor\frac{D-c_i^a}{c_i^b}\right\rfloor.\] 
Let the distinct degree bounds across all the nodes be \( d_1<d_2<\cdots<d_R,\) and group the nodes according to their degree bounds: \(V_r:=\{i\in V: d_i=d_r\},\; r=1,\ldots,R.\)
For each degree level \(d_r\), define the higher-capacity node set as
\[S_r:=\bigcup_{\ell=r}^R V_\ell
=
\{i\in V:d_i\ge d_r\}.\]
Let \(d_0:=0\) and define the degree increment as 
\[
\Delta_r:=\min (d_r-d_{r-1}, |S_r|-1).
\]

To ensure the existence of a regular graph on $S_r$, we define
\[
\widehat{\Delta}_r :=
\begin{cases}
\Delta_r, & \text{if } |S_r|\Delta_r \text{ is even},\\
\Delta_r-1, & \text{if } |S_r|\Delta_r \text{ is odd}.
\end{cases}
\]
The oracle constructs a candidate graph by sampling, for each
\(r=1,\ldots,R\), a random \(\widehat{\Delta}_r\)-regular Ramanujan graph\footnote{A $d$-regular graph $H$ is a Ramanujan graph if all the non-zero eigenvalues of its Laplacian matrix $\bm{L}_H$ lie between $d-2\sqrt{d-1}$ and $d+2\sqrt{d-1}$ \cite{hoory06}.} \(H^{(r)}\) over the node
set \(S_r\), and constructing the candidate graph as
\[
G_k=(V,E\cap E(H_k)) \mbox{ for } H_k=\bigcup_{r=1}^R H^{(r)}.
\]



\subsection{Performance Analysis}

\subsubsection{Fully-connected Base Topology}

When the base topology $G$ is a clique, the proposed graph oracle effectively produces a union of {Ramanujan graphs}~\cite{hoory06}. The choice of degrees ensures that the constructed graph satisfies the maximum per-node budget constraint. The convergence rate of the designed mixing matrix can be guaranteed as follows. 


\begin{theorem}\label{thm:rho bound on clique - unicast}
Assume $G$ is a clique, and \eqref{eq:condition on D - unicast} holds for the given cost vectors $(c_i^a, c_{ij}^b)_{(i,j)\in E}$ and budget $D$. The mixing matrix $\bm{W}$ designed by Algorithm~\ref{alg:randomized-W for unicast new} when given the layered Ramanujan graph generator as one of the oracles satisfies
\begin{align}\label{eq:rho uppper bound - unicast} 
\rho(\bm{W}) \le \left(
1-\frac{\widehat{\Delta}_1-2\sqrt{\widehat{\Delta}_1-1}}{d_R}
\right)^2, 
\end{align}
regardless of other graph oracles (valid even if $K_{\max}=1$). 
\end{theorem}

\textbf{Remark:} Compared to the solution in \cite{zhang2024energyefficient} that designs a deterministic mixing matrix defined over a single Ramanujan graph with a uniform degree ($d_1$), our layered design allows more energy-efficient nodes to have  larger degrees, which can improve the convergence rate under the same per-node budget.

\subsubsection{General Base Topology}

In the general setting with an arbitrary cost vector and base topology $G$, Algorithm~\ref{alg:randomized-W for unicast new} is expected to perform well because it addresses heterogeneity in two ways: the layered design utilizes heterogeneous energy efficiencies across nodes, while the optimized statistical mixing over a set of candidate graphs allows the algorithm to combine candidate mixing matrices with favorable spectral properties. In practice, since analytically bounding $\rho(\bm{W})$ is difficult, we use the empirical estimate 
of $\rho(\bm{W})$ as a practical surrogate for $\rho_{\mathcal{A}}^+$ in the intermediate-level optimization.

Moreover, as observed in \cite{zhang2024energyefficient}, finding a deterministic matrix $\bm{W}$ for a general graph that satisfies a given maximum per-node budget constraint under the unicast cost model and achieves $\rho(\bm{W}) < 1$ is NP-hard.
In contrast, if the problem is relaxed to allow randomized matrices as considered here, then under the corresponding budget constraint \eqref{eq: unicast K=1 constraint hetero}, the problem becomes tractable.
Specifically, Theorem~\ref{lem:finiteness} guarantees that under mild conditions on the graph oracle, Algorithm~\ref{alg:randomized-W for unicast new} guarantees $\rho(\bm{W})<1$ with $O(m)$ candidates with high probability. 
\begin{theorem}\label{lem:finiteness}
    Suppose that the base topology $G=(V,E)$ is connected, and all candidate graphs are generated by a random graph oracle $\mathcal{G}_G$, for which there exists $c>0$ such that for each $e\in E$, $G_k \sim \mathcal{G}_G$ contains $e$ with probability at least $c$. Let $\rho_k$ denote the objective value of \eqref{eq:budgeted probabilities design} when mixing $k$ such generated candidates, and $\tau$ be the minimum $k$ such that $\rho_k<1$. 
    With probability $1-e^{-\Omega(m)}$, we have $\tau=O(m/c)$.
\end{theorem}

\section{Performance Evaluation}\label{sec:Performance Evaluation}

We evaluate the proposed solution against benchmarks on a real dataset under realistic settings. 

\subsection{Evaluation Setting}

\subsubsection{Problem Setting}
We consider the standard task of image classification based on CIFAR-10, which consists of 60,000 color images in 10 classes. We train a lightweight version of ResNet-50 with {1.5M} parameters over its training dataset with 50,000 images, and then test the trained model over the testing dataset with 10,000 images. We set the learning rate to 0.01, the batch size to 64, the momentum to 0.9, and the weight decay to 0.0005. 
We employ two base topologies: (i) a $33$-node clique, which represents a densely deployed wireless network where every node can reach every other node in one hop; (ii) the topology of Roofnet~\cite{Roofnet}, which contains 33 nodes and 187 links. The data rate is set to 1 Mbps in both topologies according to \cite{Roofnet}. 

\subsubsection{Cost Parameters}
We consider two types of devices commonly used for learning in edge networks: NVIDIA Tegra X2 (TX2) with Broadcom BCM4354 and Jetson Xavier NX with Intel 8265NGW, with a computation power of $4.7$W/$6.3$W and a transmission power of $40$mW/$100$mW, 
respectively~\cite{Carbon2021}. Based on these parameters, we set the computation energy as {$c^a_i = 0.086$mWh for TX2 and $0.086$mWh for NX}, and the communication energy as {$c^b_i =0.533$mWh for TX2 and $1.333$mWh for NX}\footnote{ Our model size is $S=6$MB at FP32, batch size is $B=64$, and processing speed is $t^c_i=1.026$ms per sample for TX2 and $0.769$ms per sample for NX. Under 1Mbps, we estimate the computation energy by $c^a_i = P^c_i*B*t^c_i/3600$ mWh, and the communication energy by $c^b_i=P^t_i*S*8/1/3600$ mWh ($P^c_i$: computation power of node $i$ in W; $P^t_i$: transmission power of node $i$ in mW).}. 
We assign odd-numbered nodes to TX2 and even-numbered nodes to NX.  

\subsubsection{Benchmarks}



We separately select benchmarks for the broadcast and unicast communication settings.

\paragraph{Broadcast benchmarks.}
For broadcast communications, we compare the proposed solution with the following benchmarks:
\begin{itemize}
    \item `Vanilla D-PSGD' \cite{Lian17NIPS}, which is a baseline with all the neighbors communicating in all the iterations
    using Metropolis--Hastings weights;
    \item `AdaPC' \cite{Tchaye-Kondi24TMC}, where all the neighbors communicate periodically, with a period adapted according to \cite[Alg.~2]{Tchaye-Kondi24TMC};
    \item `BASS' \cite{Herrera25OJCS}, a state-of-the-art mixing matrix design for broadcast communication with a different objective of minimizing the communication time (by minimizing the number of collision-free transmission slots); 
    \item `Max Success' \cite{Chen23SPAWC}, another mixing matrix design for broadcast communication with a different objective of maximizing the expected number of successful links; 
    \item `SkipTrain' \cite{DeVos24IPDPSW}, which is the opposite of `AdaPC' that periodically skips local updates\footnote{For a fully connected base topology, `SkipTrain' reduces to `Vanilla D-PSGD', as there is no need for additional communications to achieve synchronization after a training round.}; 
\end{itemize}
We note that the above benchmarks subsume the performance of several other existing solutions, e.g., \cite{MATCHA22} is subsumed by `BASS' \cite{Herrera25OJCS} and \cite{Chiu23JSAC,zhang2024energyefficient} are subsumed by `Vanilla D-PSGD'. 

\paragraph{Unicast benchmarks.}
For unicast communications, we compare the proposed solution with the baselines of `Vanilla D-PSGD', `AdaPC', and `SkipTrain' that apply to both cases, as well as the following benchmarks designed specifically for unicast: \begin{itemize}
    \item `MATCHA' \cite{MATCHA22}, which aims at balancing  convergence rate and communication time by  decomposing the base topology into matchings and optimizing their activation probabilities;
    \item `Single-Layer Ramanujan' (SLR) \cite{zhang2024energyefficient}, which is the state-of-the-art design for energy efficiency in the case of fully connected base topology, by constructing a sparse Ramanujan communication graph based on the degree of the most energy-constrained node; 
    \item `Greedy' \cite{zhang2024energyefficient}, which is the state-of-the-art design for energy efficiency under a general base topology that greedily constructs a communication graph based on a given per-node energy budget by repeatedly solving an SDP.
\end{itemize}
We note that the selected benchmarks subsume the performance of  other existing solutions, e.g., \cite{Chiu23JSAC} is subsumed by `SLR' and `Greedy' from \cite{zhang2024energyefficient} under unicast. 

\subsection{Evaluation Results under Broadcast}

\subsubsection{Results for Clique}
\begin{figure}[t!]
    \centering    \centerline{\mbox{\includegraphics[height = 2in,width=.95\linewidth]{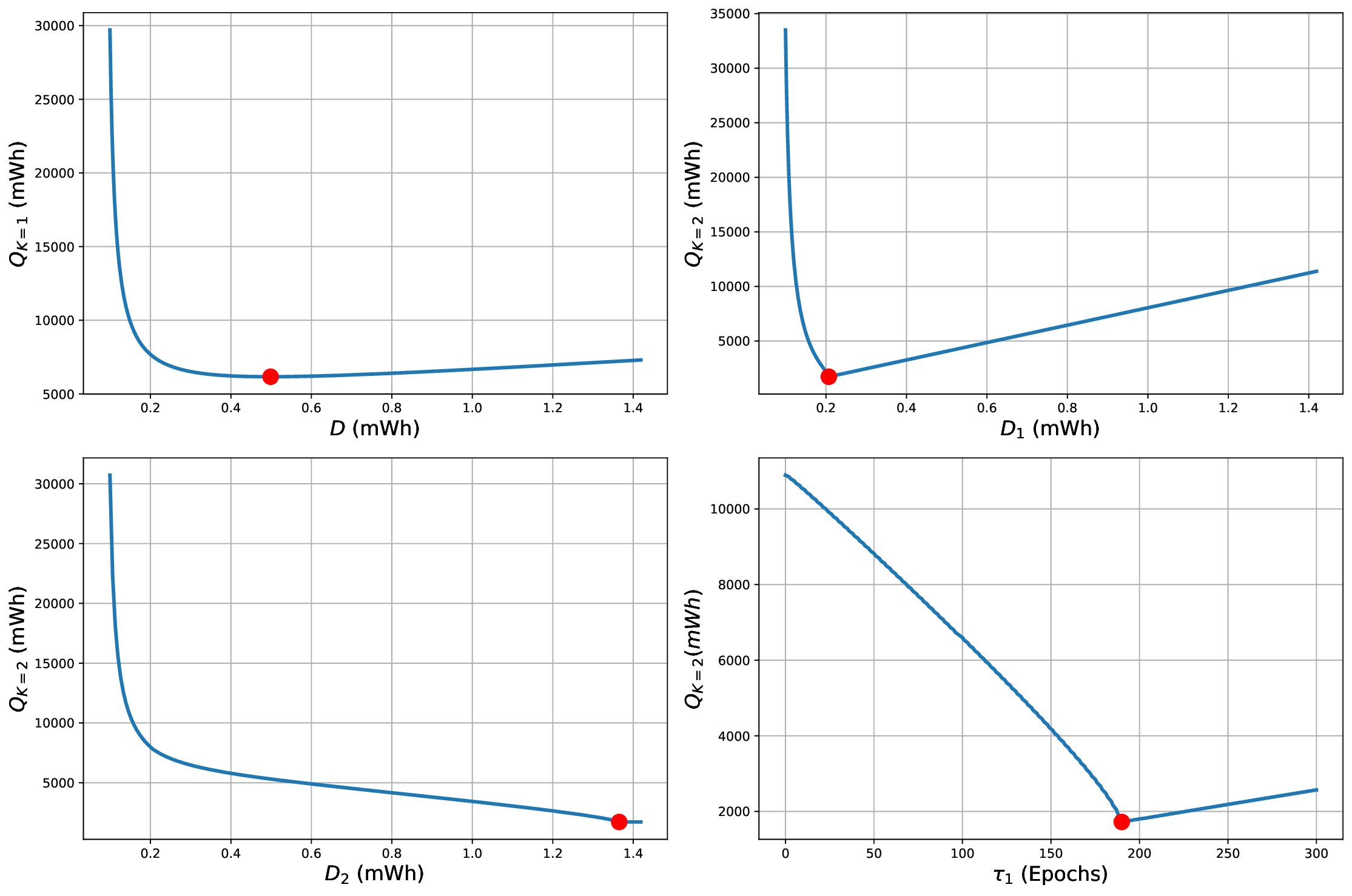}}}
    \vspace{-1em}
    \caption{(Broadcast) Design objective vs. design parameters for clique.  
    }
    \label{fig:design_clique}
    \vspace{-1em}
\end{figure}

\begin{figure}[t!]
    \centering
    \centerline{\mbox{\includegraphics[height = 2.35in,width=.95\linewidth]{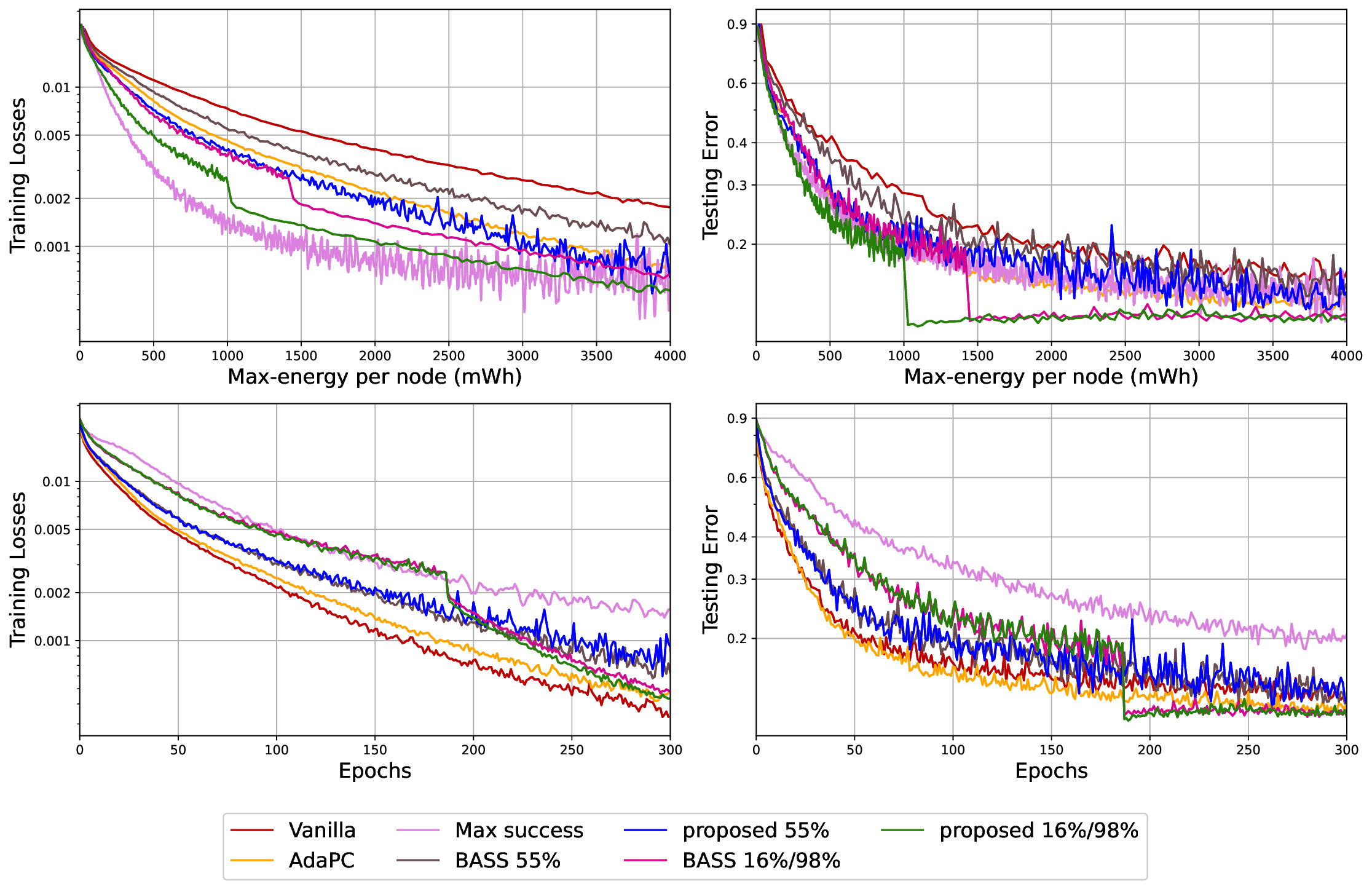}}}
    \vspace{-1em}
    \caption{(Broadcast) Training performance on clique (`proposed/BASS $x\%$': 1-phase design activating $x\%$ of nodes; `proposed/BASS $x\%$/$y\%$': 2-phase design activating $x\%$ of nodes in phase 1 and $y\%$ in phase 2). 
    }
    \label{fig:results_cifar10_clique}
    \vspace{-1em}
\end{figure}
We first evaluate the design objective $Q_K$ with respect to the budget $D_s$ and the duration $\tau_s$ for each phase. Fig.~\ref{fig:design_clique} shows the results for $K=1$  and $K=2$, while fixing the other design parameters at their optimal values. 
The optimal value for each parameter is denoted by \textcolor{red}{\textbullet}. The results not only suggest the benefit of multi-phase design as the optimal value of $Q_{K=2}$ is smaller than the optimal value of $Q_{K=1}$, but also indicate the need of switching from a lower level of activation to a higher level of activation (as $D_1<D_2$). However,  $\tau_s$ from this optimization is often larger than necessary as it is based on an upper bound on the total number of iterations. We thus normalize it by $\tau_s\cdot (T/ \overline{T})$, where $T$ is the actual number of iterations to reach convergence and $\overline{T}$ is an upper bound. 

We then evaluate the training performance in terms of  training loss and testing error, omitting `SkipTrain' as it reduces to `Vanilla' in this case. As `BASS' also has a configurable budget, we evaluate two versions of it with  the same percentage of activated nodes as our 1-phase/2-phase designs.  The results in Fig.~\ref{fig:results_cifar10_clique} show that: (i) partial activation can achieve a better energy efficiency than  activating all the nodes (`Vanilla'); (ii) the proposed solution achieves a better tradeoff between the testing error and the maximum per-node energy consumption than the benchmarks, and 2-phase design outperforms 1-phase design. 

\subsubsection{Results for Roofnet}

\begin{figure}[t!]
    \centering    \centerline{\mbox{\includegraphics[height = 2in,width=.95\linewidth]{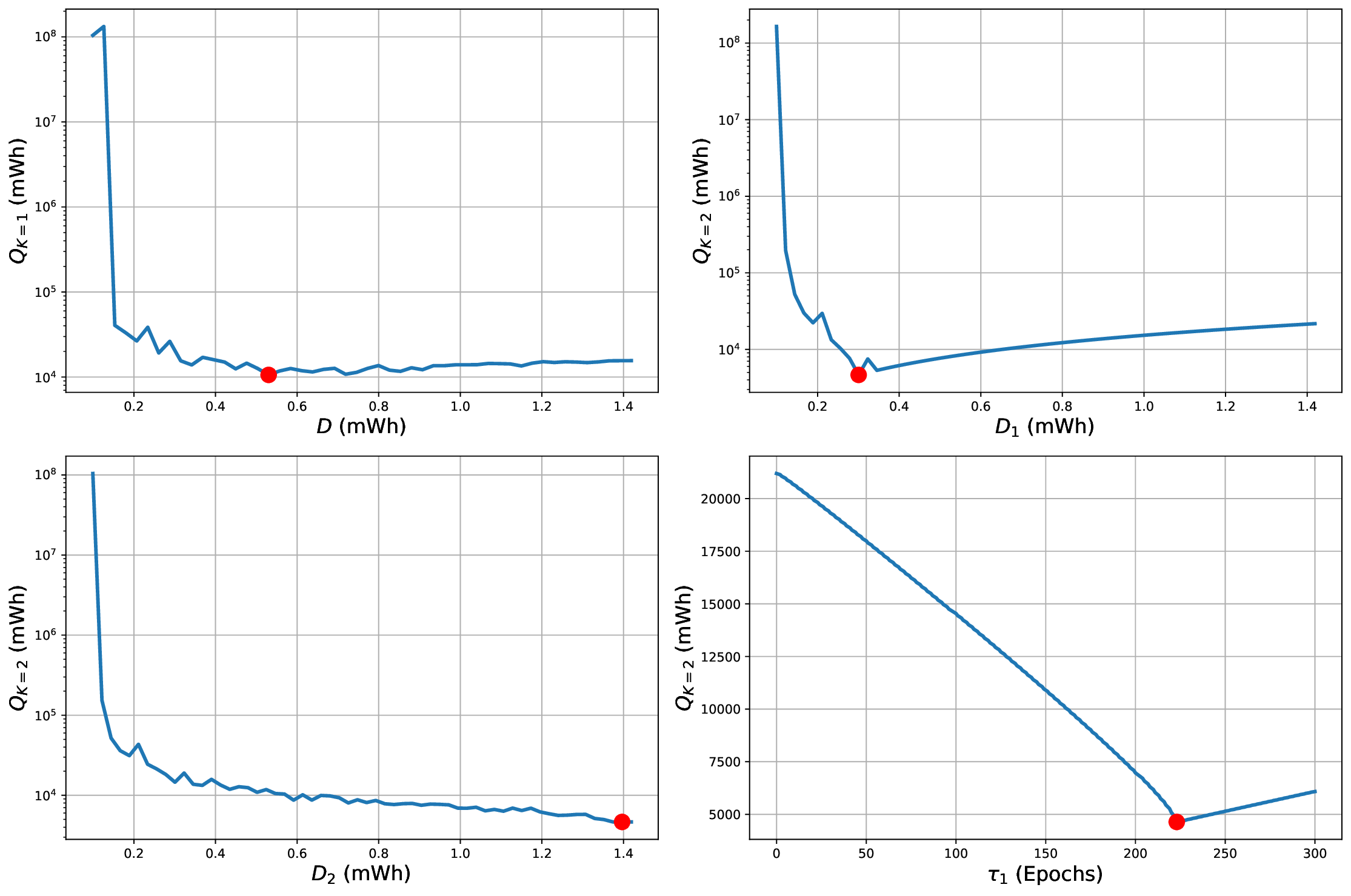}}}
    \vspace{-1em}
    \caption{(Broadcast) Design objective vs. design parameters for Roofnet. 
    }
    \label{fig:design_roofnet}
    \vspace{-1em}
\end{figure}

\begin{figure}[t!]
    \centering    \centerline{\mbox{\includegraphics[height = 2.35in,width=1\linewidth]{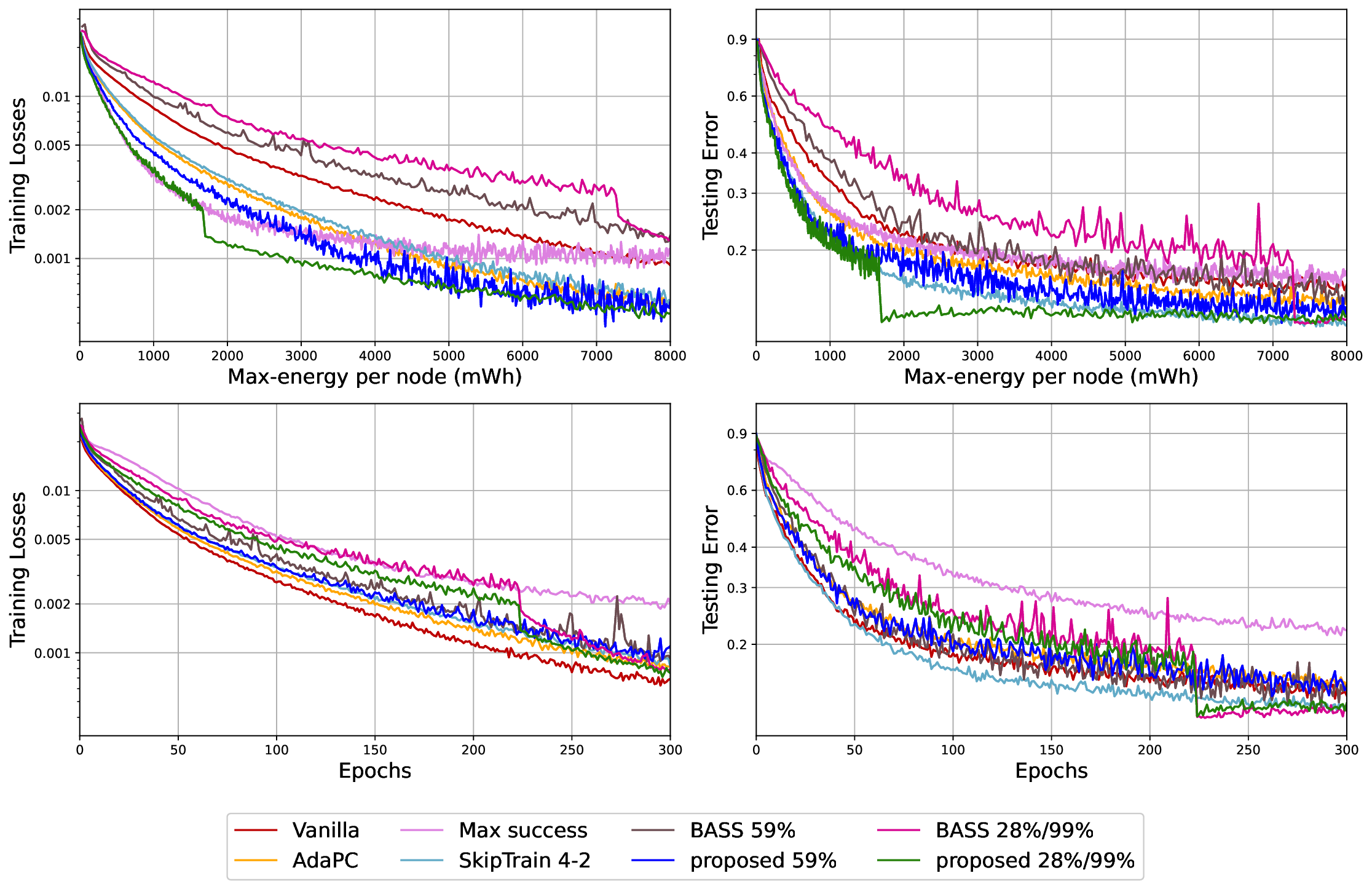}}}
    \vspace{-1em}
    \caption{(Broadcast) Training performance on Roofnet. }
    \label{fig:results_cifar10_roofnet}
    \vspace{-1em}
\end{figure}
We repeat the experiments on Roofnet, where `SkipTrain' is configured according to the recommendation by \cite{DeVos24IPDPSW} for a topology of similar average degree.  The results in Fig.~\ref{fig:design_roofnet}--\ref{fig:results_cifar10_roofnet} show similar observations as Fig.~\ref{fig:design_clique}--\ref{fig:results_cifar10_clique}, except that (i) the design objective in Fig.~\ref{fig:design_roofnet} is based on numerically estimated $\rho$ values as explained in Section~\ref{subsubsec:Discussion on General Base Topology} (from $500$ matrices per budget), (ii) `BASS' has a bigger performance gap with our solution due to its negligence of  balancing energy consumption across nodes, and (iii) it takes more epochs and energy consumption to reach convergence due to the limited connectivity between nodes. 

\subsection{Evaluation Results under Unicast}
\begin{figure}[t!]
    \centering    \centerline{\mbox{\includegraphics[height = 2in,width=.95\linewidth]{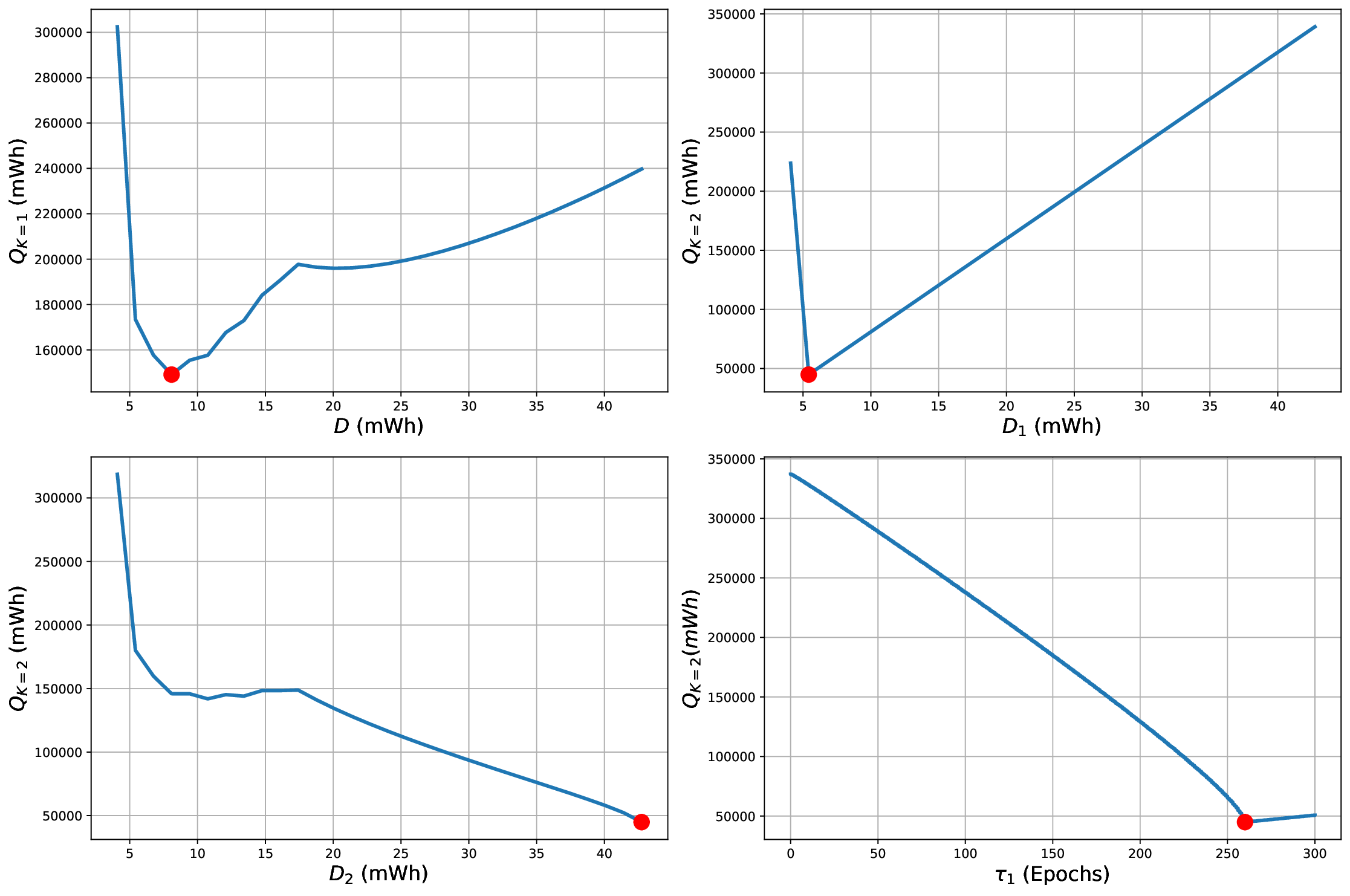}}}
    \vspace{-1em}
    \caption{(Unicast) Design objective vs. design parameters for clique.  
    }
    \label{fig:design_unicast_clique}
    \vspace{-1em}
\end{figure}

\begin{figure}[t!]
    \centering
    \centerline{\mbox{\includegraphics[height = 2.35in,width=.95\linewidth]{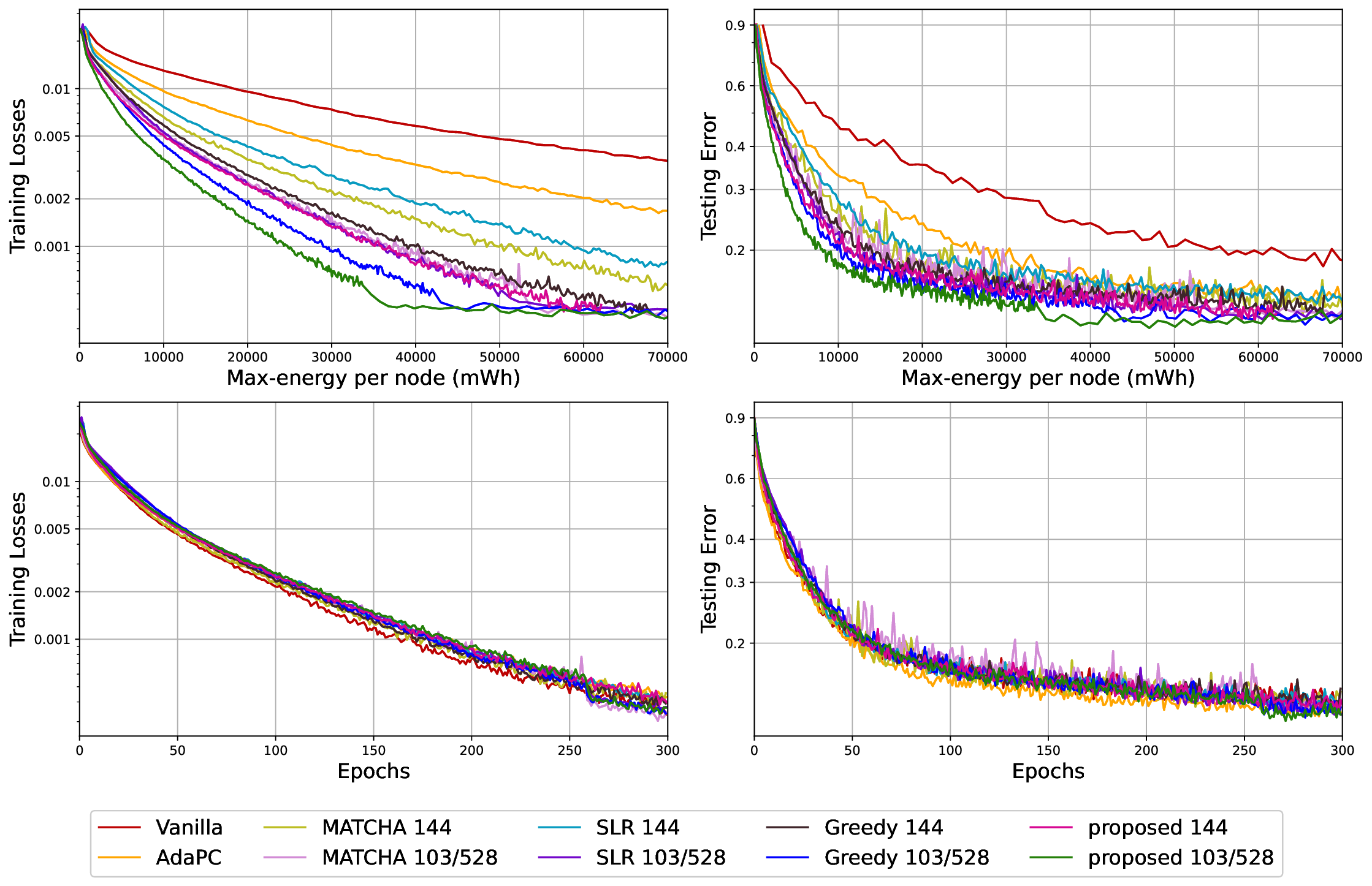}}}
    \vspace{-1em}
    \caption{(Unicast) Training performance on clique. (`proposed/MATCHA/SLR/Greedy $x$': 1-phase design activating $x$ links; `proposed/MATCHA/SLR/Greedy  $x$/$y$': 2-phase design activating $x$ links in phase 1 and $y$ links in phase 2).
    }
    \label{fig:results_unicast_cifar10_clique}
    \vspace{-1em}
\end{figure}

\begin{figure}[t!]
    \centering    \centerline{\mbox{\includegraphics[height = 2in,width=.95\linewidth]{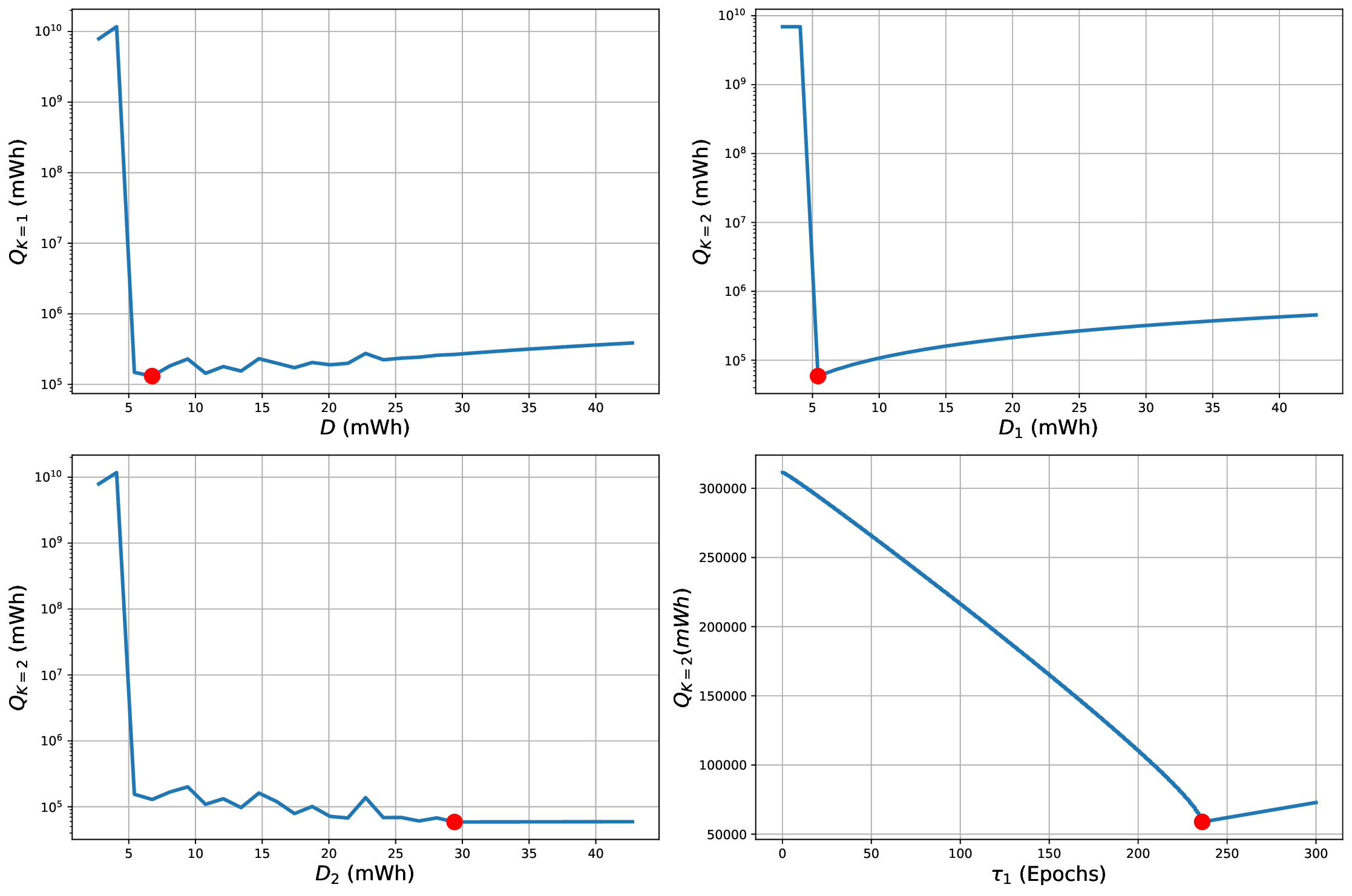}}}
    \vspace{-1em}
    \caption{(Unicast) Design objective vs. design parameters  for Roofnet.  
    }
    \label{fig:design_unicast_roofnet}
    \vspace{-1em}
\end{figure}

\begin{figure}[t!]
    \centering    \centerline{\mbox{\includegraphics[height = 2.35in,width=1\linewidth]{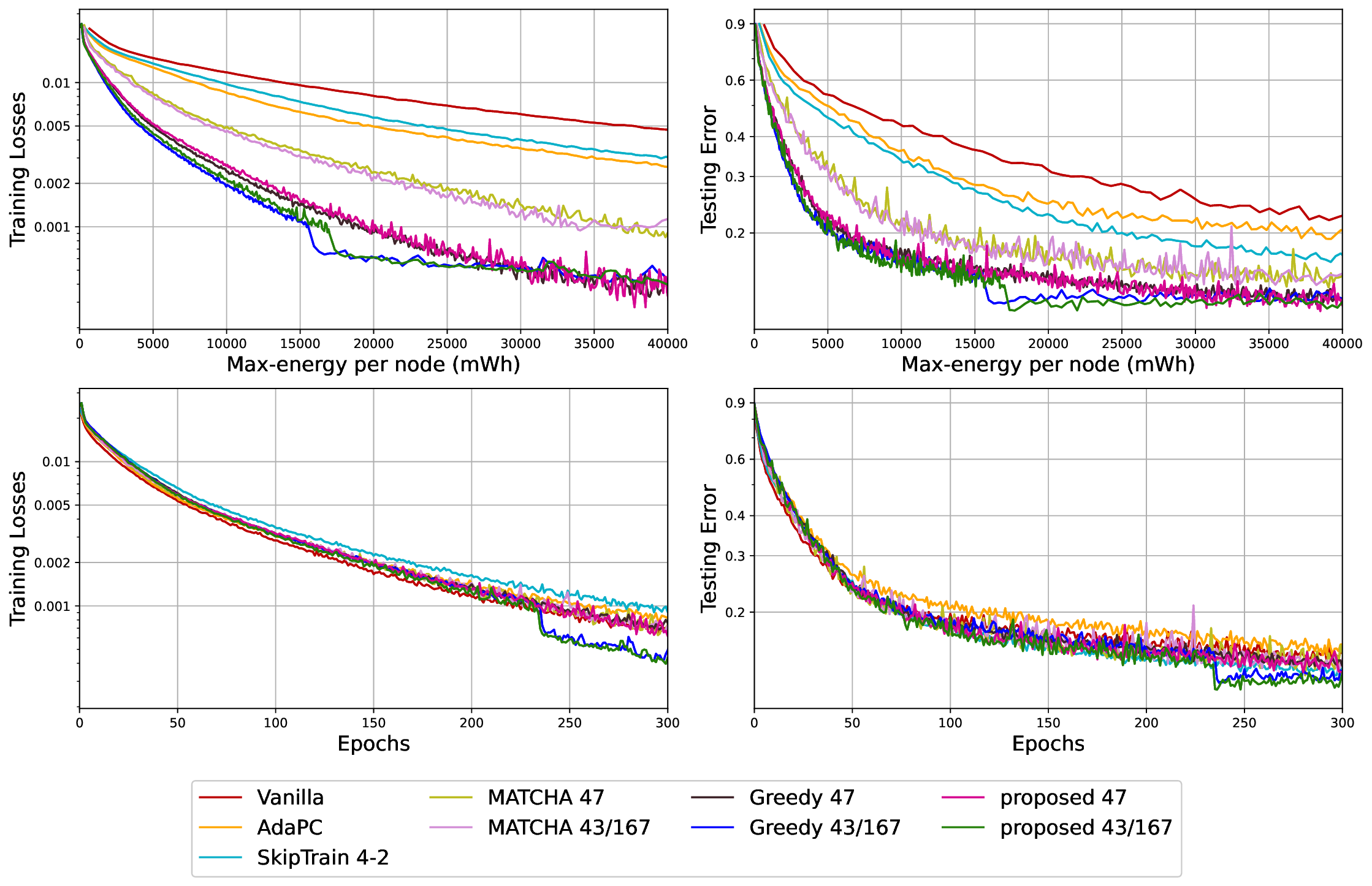}}}
    
    \vspace{-1em}
    \caption{(Unicast) Training performance on Roofnet.}

    \label{fig:results_unicast_cifar10_roofnet}
    \vspace{-1em}
\end{figure}
\subsubsection{Results for Clique}
We follow the same design objective evaluation procedure as in the broadcast setting. Fig. ~\ref{fig:design_unicast_clique} presents the objective $Q_K$ with respect to the communication budget $D_s$ and phase duration $\tau_s$. Similar to broadcast results, the optimal value of $Q_{K=2}$ is smaller than that of $Q_{K=1}$, demonstrating the benefit of multi-phase design. The optimal configuration for $Q_{K=2}$ also implies a switch from a lower communication budget in the first phase to a higher budget in the second phase ($D_1 \leq D_2$).

Since the unicast baselines `SLR', `MATCHA', and `Greedy' have configurable communication budgets, we evaluate two versions of each baseline using the same total number of activated links as our 1-phase and 2-phase designs. To demonstrate the benefit of the probabilistic mixing matrix design in \eqref{eq:budgeted probabilities design}, we use both the Greedy graph construction from \cite{zhang2024energyefficient} and our proposed layered graph construction as graph oracles for generating candidate communication graphs. We set $K_{\max}=6$ for both topologies (1 Greedy + 5 layered Ramanujan graphs).

The results in Fig.~\ref{fig:results_unicast_cifar10_clique} show trends similar to those observed in the broadcast setting: (i) the general baselines `Vanilla', `AdaPC', and `SkipTrain' (reduced to `Vanilla' in this case) perform worse than the baselines specifically designed for unicast communication; (ii) the proposed solution achieves a better tradeoff than  the baselines, with the 2-phase design outperforming the 1-phase design. 


\subsubsection{Results for Roofnet}

We repeat the experiments on Roofnet. The results in Fig. ~\ref{fig:design_unicast_roofnet}-~\ref{fig:results_unicast_cifar10_roofnet} exhibit trends similar to those observed as Fig. ~\ref{fig:design_unicast_clique}-~\ref{fig:results_unicast_cifar10_clique}, with two notable differences. 
First, due to the lack of an analytical upper bound on $\rho(\bm{W})$ under a general base topology,  the design objective in Fig.~\ref{fig:design_unicast_roofnet}  is estimated empirically: for each communication budget, we run Algorithm~\ref{alg:randomized-W for unicast new} with 
a number of repetitions (set to 14) of the proposed layered graph construction oracle, and then use the optimized objective value in \eqref{eq: unicast K=1 obj clique} as the empirical estimate of $\rho(\bm{W})$. 
Second, `Greedy' achieves performance comparable to the proposed solution and outperforms the other baselines. 
This is because for Roofnet, the probabilistic mixing design assigns substantial probabilities to the mixing matrices generated by the Greedy oracle (68\% in 1-phase design and 73\%/99\% in the first/second phase of 2-phase design), in contrast to 0\% in the case of clique. This further demonstrates the benefit of the probabilistic mixing design in \ref{eq:budgeted probabilities design}, which can combine the strengths of different graph-construction oracles to achieve a consistently good tradeoff for diverse base topologies. 

Meanwhile, we note that the Greedy oracle from \cite{zhang2024energyefficient} is computationally expensive due to its requirement of repeatedly solving semidefinite programs. In this regard, our proposed layered graph construction oracle provides not only a closed-form $\rho(\bm{W})$ upper bound in the special case (i.e., clique) but also a computationally efficient method to estimate this value in the general case, which is crucial for solving the upper- and intermediate-level optimizations. 

\section{Conclusion}\label{sec:Conclusion}

We considered the mixing matrix design to minimize the maximum per-node energy consumption in DFL under both broadcast and unicast communications.
Based on a novel convergence theorem that allows arbitrarily time-varying mixing matrices, we proposed a multi-phase design framework that designs randomized mixing matrices as well as how long to use each matrix. Our evaluations on real data demonstrated that the proposed design can achieve a superior tradeoff between maximum per-node energy consumption and accuracy even with two phases, thus improving the energy efficiency of learning in wireless  networks. \looseness=0

\section*{Acknowledgments}
The first author was supported in part by EPSRC New Investigator Award UKRI155.

\bibliographystyle{IEEEtran}
\bibliography{references}
\appendix
\section{Appendix}\label{sec: support proof}
\subsection{Proof of Theorem~\ref{thm: new convergence bound nonconvex}}

We use the following notation for the ``consensus distance'' at iteration $t$:
\[\Xi_t:=\frac{1}{m}\E \sum_{i=1}^m \|x_i^{(t)} - \overline{\bm{x}}^{(0)} \|^2.\]
We use $\E_{t+1}[\cdot]$ to denote the conditional expectation $\E[\cdot \mid \{\bm{x}^{(t)}_i\}]$.
Also, we let $\bm{x}^*$ denote the optimal parameter vector of $F(\cdot)$ and let $F_{\inf}:=F(x^*)$.

\begin{lemma}[Lemma 11, \cite{Koloskova20ICML}]
    \label{lem : 1  step descent contraction nonconvex}
    Under assumptions (1)-(3), D-PSGD with stepsize 
    $\eta < \frac{1}{4L(M_1+1)}$
    satisfies 
    \begin{align*}
        \E_{t+1}  [F(\overline{\bm{x}}^{(t+1)})] 
        \le F(\overline{\bm{x}}^{(t)}) - \frac{\eta}{4} 
        \| \nabla F(\overline{\bm{x}}^{(t)}) \|_2^2 
     +\eta L^2\Xi_t
    + \frac{\eta^2 L \hat{\sigma}^2}{m} .
    \end{align*}
\end{lemma}

\begin{lemma}
    \label{lem: 2 step consensus bound nonconvex}
    Under assumptions (1)-(3), D-PSGD with the stepsize $\eta \le \frac{p^{(t-1)}}{8L\sqrt{12+2p^{(t-1)}M_1}}$ satisfies
    \begin{align*}
        \Xi_t &\le \big(1-\frac{p^{(t-1)}}{2}\big)\Xi_{t-1} + 
       \\ 
       &2\eta^2 \left[  \big(\frac{6}{p^{(t-1)}} + M_1\big) \big( M_2 \|\nabla F(\overline{\bm{x}}^{(t-1)}) \|^2 
       +\hat{\zeta}^2
       \big)+ \hat{\sigma}^2 
                \right].
    \end{align*}
\end{lemma}

Noting that the parameter $p^{(t-1)}$ in \cite{Koloskova20ICML} is originally defined in a different form, but due to Lemma~3.1 in \cite{zhang2024energyefficient}, it can be treated the same, in the case of $\tau=1$.
Then Lemma~\ref{lem: 2 step consensus bound nonconvex} follows from Lemma 12 in \cite{Koloskova20ICML} by setting $\tau=1$ and improving some constants in this special case. 

\begin{proof}[Proof of Theorem~\ref{thm: new convergence bound nonconvex}]
    We define the following notations for ease of the proof:
    \begin{itemize}
        \item $f_t:=\E[F(\overline{\bm{x}}^{(t)})] -F_{\inf}$.
        \item $e_t:= \E \|\nabla F(\overline{\bm{x}}^{(t)}) \|^2$.
        \item $\nu_j^t:=\prod_{i=j}^{t-1} (1-\frac{p_i}{2})$.
        \item $\kappa := \frac{\sqrt{(1+M_1) (1+M_2)}}{p_{\min}}$.
        \item $\Upsilon := M_2\max_j   \left( 
                \frac{6\pi_j}{p^{(j)}} +M_1\pi_j
            \right)$.
    \end{itemize}
    In these notations, Lemma~\ref{lem : 1  step descent contraction nonconvex} implies 
    \begin{equation}
        \label{eq: ft recursion}
        f_{t+1} \le f_t -\frac{\eta}{4}e_t + \frac{L\hat{\sigma}^2\eta^2}{m} + 
        \eta L^2 \Xi_t,
    \end{equation}
    and Lemma~\ref{lem: 2 step consensus bound nonconvex} upper bounds $\Xi_t$ by
    \begin{align}
        \label{eq: nonconvex consensus recursion}
         \Xi_t &\le \big(1-\frac{p^{(t-1)}}{2}\big)\Xi_{t-1} + 
       2\eta^2 \cdot  \\
       &\left[  \big(\frac{6}{p^{(t-1)}} + M_1\big) \big( M_2 e_{t-1} 
       +\hat{\zeta}^2
       \big)+ \hat{\sigma}^2 
                \right]. \nonumber
    \end{align}
    We un-roll the recursion in \eqref{eq: nonconvex consensus recursion} to obtain
    \begin{align*}
        \Xi_t\le \nu_0^t\Xi_0 +2\eta^2\sum_{j=0}^{t-1}\nu_{j+1}^t 
        \left[ 
         \big(\frac{6}{p^{(j)}} + M_1\big) \big( M_2 e_{j} 
       +\hat{\zeta}^2
       \big)+ \hat{\sigma}^2 
                \right].
    \end{align*}
    Noting that $\pi_j = \sum_{t>j} \nu_{j+1}^{t}$,
    we sum over $t$ up to $T$ and get
    \begin{align*}
        \sum_{t=1}^T \Xi_t &\le \sum_{t=1}^T \nu_0^t \Xi_0 +
        2\eta^2 \sum_{t=1}^T \sum_{j=0}^{t-1} \nu_{j+1}^t\cdot \\
        &\left[ 
         \big(\frac{6}{p^{(j)}} + M_1\big) \big( M_2 e_{j} 
       +\hat{\zeta}^2
       \big)+ \hat{\sigma}^2 
                \right]\\
        & \le \sum_{t=1}^{T} \nu_0^t \Xi_0 +  
        12\eta^2\hat{\zeta}^2 \sum_{j=0}^{T-1}\sum_{t=j+1}^T 
        \frac{\nu_{j+1}^t}{p^{(j)}}  \\
        & +  2\eta^2(\hat{\sigma}^2 + M_1 \hat{\zeta}^2)
        \sum_{j=0}^{T-1} \sum_{t=j+1}^T \nu_{j+1}^t\\
        & + 2\eta^2 
        \sum_{j=0}^{T-1} M_2 \left( 
                \frac{6}{p^{(j)}} +M_1
            \right) e_j \sum_{t=j+1}^T \nu_{j+1}^t\\
        &\le (1+\pi_0) \Xi_0  + 12\eta^2 \hat{\zeta}^2 
        \sum_{j=0}^{T-1} \frac{\pi_j}{p_j} \\
        &+ 2\eta^2 \sum_{j=0}^{T-1} (\hat{\sigma}^2 + M_1 \hat{\zeta}^2)\pi_j \\ 
        &+ 2\eta^2 \sum_{j=0}^{T-1} M_2 \left( 
                \frac{6}{p^{(j)}} +M_1
            \right)e_j \pi_j \\
        &\le (1+\pi_0) \Xi_0 + 2\eta^2 \Upsilon \sum_{j=0}^{T-1} e_j\\
        &+   2\eta^2 T
        [(\hat{\sigma}^2 + M_1 \hat{\zeta}^2) \Pi_1(T) + 6\hat{\zeta}^2 \Pi_2(T)].
    \end{align*}
    Next we move $f_t$ in \eqref{eq: ft recursion} to the right-hand side and average it over time to get
    \begin{equation}
        \label{eq:timeavg}
        \frac{1}{8T} \sum_{t=0}^{T-1} e_t
        \le \frac{1}{T} \sum_{t=0}^{T-1} \frac{f_t - f_{t+1}}{\eta} +
        \frac{\eta L \hat{\sigma}^2}{m} +
        \frac{L^2}{T} \sum_{t=0}^{T-1} \Xi_t. 
    \end{equation}
    Now we plug in the bound for $ \sum_{t=1}^T \Xi_t $  to \eqref{eq:timeavg} and obtain
    \begin{align}
          \sum_{t=0}^{T-1} \frac{e_t}{8T} &\le 
         \frac{f_0 }{T\eta } +    \frac{\eta L \hat{\sigma}^2}{m}
         + \frac{L^2(2+\pi_0)\Xi_0}{T} 
         +  \frac{2L^2 \eta^2 \Upsilon}{T} \sum_{j=0}^{T-1}e_j \label{eq: step sum bound}\\
        &+2\eta^2 L^2
        [(\hat{\sigma}^2 + M_1 \hat{\zeta}^2) \Pi_1(T) + 6\hat{\zeta}^2 \Pi_2(T)].
    \end{align}
    We let $\alpha(T):=2L^2[(\hat{\sigma}^2 + M_1 \hat{\zeta}^2) \Pi_1(T) + 6\hat{\zeta}^2 \Pi_2(T)]$.
    We need the learning rate to satisfy the conditions in Lemmas~\ref{lem : 1  step descent contraction nonconvex}-\ref{lem: 2 step consensus bound nonconvex}.
    By choosing a sufficiently small learning rate\footnote{While the choice of $\eta$ in \cite{Koloskova20ICML} does not explicitly include the term $(4L(M_1+1))^{-1}$, we believe that, without this additional term, the condition required by Lemma \ref{lem : 1 step descent contraction nonconvex} may not hold in general.} 
    \[
    \eta=\min\left\{ 
        (\frac{m f_0}{\hat{\sigma}^2 L T})^{\frac{1}{2}},
        (\frac{f_0}{T \alpha(T)})^{\frac{1}{3}}, 
        \frac{1}{20L\kappa},
        \frac{1}{4L(M_1+1)}
    \right\}, 
    \]
    the learning rate satisfies the conditions in Lemma~\ref{lem : 1  step descent contraction nonconvex} and \ref{lem: 2 step consensus bound nonconvex}. 
    Moreover, we claim that $\eta ^2 \le \frac{1}{32L^2\Upsilon}$. 
    Indeed, 
    \begin{align*}
        \eta^2 &\le \frac{1}{400L^2\kappa^2 }
        \le \min_j \frac{(p^{(j)})^2}{32L^2(M_2+1)(12+2M_1 p^{(j)})} \\
        &\le \frac{1}{32L^2\max_j\{
            \frac{2M_1 M_2}{p^{(j)}}
            + \frac{12 M_2}{(p^{(j)})^2}
        \}}
    \end{align*}
     Let $j^*$ denote the index of $j$ that maximizes $ \max_j   \left( \frac{6\pi_j}{p^{(j)}} +M_1\pi_j\right)$, that is, the index of $j$ which yields $\Upsilon$.
     By Lemma~\ref{fact: pi}, $\pi_{j^*} \le \max_j \frac{2}{ p^{(j)}}$, so
     \begin{align*}
         \max_j\left\{
            \frac{2M_1 M_2}{p^{(j)}}
            + \frac{12 M_2}{(p^{(j)})^2}
        \right\}
        &\ge M_1 M_2 \pi_{j^*} + \max_j \frac{6M_2 \pi_{j^*}}{p^{(j)}}\\
        &\ge M_1 M_2  \pi_{j^*} + \frac{6M_2 \pi_{j^*}}{p^{(j^*)}} = \Upsilon.
     \end{align*}
     Then it follows that $\eta ^2 \le \frac{1}{32L^2\Upsilon}$ and 
     \begin{equation}
        \label{eq: step a}
         \frac{2L^2 \eta^2 \Upsilon}{T} \sum_{j=0}^{T-1}e_j \le \sum_{j=0}^{T-1}\frac{e_j}{16}.
     \end{equation}
     In addition, since $\eta \le \min\{  (\frac{m f_0}{\hat{\sigma}^2 L T})^{\frac{1}{2}},
        (\frac{f_0}{T \alpha(T)})^{\frac{1}{3}}\}$, we obtain
    \begin{equation}
        \label{eq: step b}
        \frac{\eta L \hat{\sigma}^2}{m} 
        \le (\frac{\hat{\sigma}^2 L f_0}{mT})^{\frac{1}{2}}
    \end{equation}
    and \begin{equation}
        \label{eq: step c}
        \eta^2 \alpha(T) \le 
        \left( 
            \frac{f_0^2 \alpha(T)}{T^2} 
        \right)^{\frac{1}{3}}
    \end{equation}
    respectively.
    Finally, 
    \begin{align}
        &\frac{f_0 }{T\eta} 
        =\frac{f_0}{T}\cdot\max\left\{
        (\frac{\hat{\sigma}^2 L T}{m f_0})^{\frac{1}{2}},
        (\frac{T\alpha(T)}{f_0})^{\frac{1}{3}}, 
        20L\kappa,
        4L(M_1+1)
        \right\} \\
        &= \max \left\{ 
            (\frac{\hat{\sigma}^2 L f_0}{mT})^{\frac{1}{2}},
            (\frac{f_0^2 \alpha(T) }{T^2})^{\frac{1}{3}},
            \frac{20L\kappa f_0}{T},
            \frac{4f_0 L (M_1+1)}{T}
        \right\}. \label{eq: step d}
    \end{align}
    Combining inequalities \eqref{eq: step sum bound}–\eqref{eq: step d} implies that
    \begin{align*}
        \sum_{t=0}^{T-1} \frac{e_t}{16T} &\le 
        2(\frac{\hat{\sigma}^2 L f_0}{mT})^{\frac{1}{2}}
        + 2(\frac{f_0^2 \alpha(T) }{T^2})^{\frac{1}{3}}
        + \frac{L^2(2+\pi_0)\Xi_0}{T} \\
        &+  \frac{20L\kappa f_0}{T}  + \frac{4f_0 L (M_1+1)}{T}.
    \end{align*}
    The theorem follows by setting the right-hand side to be at most $\epsilon/16$.
\end{proof}

\subsection{Convergence Theorem for Convex Objectives} \label{sec: convex convergence}
The convergence of D-PSGD for convex objectives is established under the following assumptions:
\begin{enumerate}[(1)]
\setcounter{enumi}{3}   
    \item Each local objective function $F_i(\boldsymbol{\bm{x}})$ is $L$-Lipschitz smooth and convex,  i.e., for $x, x' \in \mathds{R}^d$, 
    \[
        F_i(x)-F_i(x')\le 
        \langle \nabla F_i(x), x-x' \rangle.
    \]
\item There exist a constant $\hat{\sigma}>0$ such that 
    $$ {1\over m}\sum_{i\in V}\E[\|g(\bm{x^*};\xi_i)-\nabla F_i(\bm{x^*}) \|^2] \leq \hat{\sigma}^2.$$
    \item There exist a constant $\hat{\zeta}>0$ such that 
    $${1\over m}\sum_{i\in V}\|\nabla F_i(\bm{x^*})\|^2\leq \hat{\zeta}^2.$$ 
\end{enumerate}
While condition (4) is stronger than condition (1), the additional convexity assumption allows conditions (5)--(6) to be significantly weaker than conditions (2)--(3): not only can we set $M_1 = M_2 = 0$, 
but it also suffices to impose the two noise bounds only at $x=x^*$. 
    \begin{theorem}\label{thm: new convergence bound}
    The D-PSGD under assumptions (4)--(6) can achieve $\epsilon$-convergence (i.e., $\frac{1}{T} \sum_{t=0}^{T-1} (\E [F(\overline{\bm{x}}^{(t)})] - F_{\inf} )\le \epsilon$) when the number of iterations $T$ satisfies 
    $T\geq T_4(\Pi_1(T), \Pi_2(T),  \pi_0, p_{\min},  \epsilon,  \overline{\bm{x}}^{(0)})$ 
    for 
    \begin{align}
        &T_4(\Pi_1, \Pi_2, \pi_0, p_{\min},  \epsilon,   \overline{\bm{x}}^{(0)}):= r_0 \cdot \\
    &O\left({\hat{\sigma}^2\over m\epsilon^2} +{\sqrt{(\hat{\sigma}^2 \Pi_1 +\hat{\zeta}^2 \Pi_2) L} \over  \epsilon^{3/2}}+{(1+\pi_0) L \Xi_0 \over \epsilon}
    +{L \over \epsilon p_{\min}}
    \right),
    \end{align}
    where $r_0:=\|\overline{\bm{x}}^{(0)} -\bm{x}^*\|^2$ ($\bm{x}^*$ denotes the optimal parameter vector) and
    $\Xi_0:=\frac{1}{m}\sum_{i=1}^m \|x_i^{(0)} - \overline{\bm{x}}^{(0)} \|^2$.
\end{theorem}
Convergence analyses of gossip-based algorithms on static communication graphs, such as Decentralized Gradient Descent (DGD) \cite{YuanLingYin16}, are well established.
For time-varying topologies, D-PSGD \cite{Koloskova20ICML} establishes convergence under  periodic communication patterns, while gradient-tracking variants have been analyzed in \cite{NedicOlshevsky15} and \cite{NedicOlshevskyShi17} under stronger assumptions.
Our approach is most closely related to \cite{Koloskova20ICML}, but it removes the periodicity assumption and applies to arbitrary time-varying graphs.

\begin{lemma}[Lemma 8, \cite{Koloskova20ICML}]
    \label{lem : 1 step descent contraction convex}
    Under assumptions (4)-(6), D-PSGD with the stepsize $\eta \le \frac{1}{12L}$ satisfies
    \begin{align*}
         \E_{t+1}& [\| \overline{\bm{x}}^{(t+1)} - x^* \|^2]  
        \le 
        \| \overline{\bm{x}}^{(t)} - x^* \|^2 \\
        & -\eta( F(\overline{\bm{x}}^{(t)}) - F^*) + \frac{\hat{\sigma}^2\eta^2}{m} + 
        3\eta L \Xi_t,
    \end{align*}
    when local functions $F_i's$ are convex.
\end{lemma}

\begin{lemma}
    \label{lem: 2 step consensus bound convex}
    Under assumptions (4)-(6), D-PSGD with the stepsize $\eta \le \frac{p^{(t-1)}}{400L}$ satisfies
    \begin{align*}
        \Xi_t &\le \big(1-\frac{p^{(t-1)}}{2}\big)\Xi_{t-1} + 
       \eta^2 
       \cdot \\
       &\left[ 
            \frac{72L }{p^{(t-1)}}\E \big( F(\overline{\bm{x}}^{(t-1)}) - F_{\inf}\big) + 8 \hat{\sigma}^2 +\frac{18 \hat{\zeta}^2}{p^{(t-1)}}\right],
    \end{align*}
    when local functions $F_i$'s are convex.
\end{lemma}
Analogous to Lemma~\ref{lem: 2 step consensus bound nonconvex}, Lemma~\ref{lem: 2 step consensus bound convex} follows directly from Lemma 3.1 in \cite{zhang2024energyefficient} and Lemma 8 in \cite{Koloskova20ICML} by setting 
$\tau=1$ and appropriately adjusting the constants.

\begin{proof}[Proof of Theorem~\ref{thm: new convergence bound}]
    We introduce several notations to simplify the expressions:
    \begin{itemize}
        \item $r_t:=\E[ \| \overline{\bm{x}}^{(t)} - x^* \|^2]$.
        \item $f_t:= \E F(\overline{\bm{x}}^{(t)}) - F^*$.
        \item  $\Upsilon:=\sqrt{\max_{j} (\pi_j/p^{(j)})}$.
    \end{itemize}
    To control the error at each $t$, we first apply the standard 1-step descent contraction by Lemma~\ref{lem : 1 step descent contraction convex} and obtain
    \begin{equation}
        \label{eq:recursion for r}
        r_{t+1} \le r_t -\eta f_t + \frac{\hat{\sigma}^2\eta^2}{m} + 3\eta L\Xi_t.
    \end{equation}
    Next we apply the upper bound on the consensus distance by Lemma~\ref{lem: 2 step consensus bound convex} for every $t\ge 1$:
    \begin{equation}\label{eq:consensus recursion}
        \Xi_t \le (1-\frac{p^{(t-1)}}{2})\Xi_{t-1} +\eta^2 \left[ 
            \frac{72L f_{t-1}}{p^{(t-1)}} + 8 \hat{\sigma}^2 +\frac{18 \hat{\zeta}^2}{p^{(t-1)}}
        \right].
    \end{equation}
    We un-roll the recursion in \eqref{eq:consensus recursion} to obtain
    \begin{align*}
        \Xi_t &\le \left[ \prod_{j=0}^{t-1} (1-\frac{p^{(j)}}{2}) \right] \Xi_0  
        + 72\eta^2 L \sum_{j=0}^{t-1} \left[ 
          \frac{f_j}{p^{(j)}}  \prod_{i=j+1}^{t-1} (1-\frac{p^{(i)}}{2}) 
        \right] \\
        &+\eta^2 \sum_{j=1}^{t-1} \left[ 
            (8\hat{\sigma}^2 + \frac{18\hat{\zeta}^2}{p^{(j)}}) \cdot \prod_{i=j+1}^{t-1} (1-\frac{p^{(i)}}{2}) 
        \right].
    \end{align*}
    We sum up $\Xi_t$ for $t=1,\dots, T$.
    \begin{align*}
        \sum_{t=1}^T \Xi_t &\le \left[ \sum_{t=1}^T\prod_{j=0}^{t-1} (1-\frac{p^{(j)}}{2}) \right] \Xi_0 \\
        &+ 72\eta^2 L \sum_{t=1}^T \sum_{j=0}^{t-1}\left[ 
          \frac{f_j}{p^{(j)}}  \prod_{i=j+1}^{t-1} (1-\frac{p^{(i)}}{2}) 
        \right] \\
        &+ 18\eta^2  \sum_{t=1}^T \sum_{j=0}^{t-1} \left[ 
            (\hat{\sigma}^2 + \frac{\hat{\zeta}^2}{p^{(j)}}) \prod_{i=j+1}^{t-1} \big(1-\frac{p^{(i)}}{2}\big) 
        \right]\\
        &\le \left[ \sum_{t=1}^T\prod_{j=1}^{t-1} (1-\frac{p^{(j)}}{2}) \right] \Xi_0 \\
        &+72\eta^2L \left[  
            \sum_{j=0}^{T-1} \frac{f_j}{p^{(j)}} \left( 
                \sum_{T\ge t>j} \prod_{i=j+1}^{t-1} \big(1-\frac{p^{(i)}}{2}\big)
            \right) 
        \right] \\
        &+ 18\eta^2 \sum_{j=0}^{T-1} \left[ 
             (\hat{\sigma}^2 + \frac{\hat{\zeta}^2}{p^{(j)}}) \sum_{T\ge t>j} \prod_{i=j+1}^{t-1} (1-\frac{p^{(i)}}{2})
        \right] \\
        &\le \pi_0 \Xi_0 
        + 72\eta^2 L \sum_{j=0}^{T-1} \frac{f_j\pi_j}{p^{(j)}}  +18\eta^2 \sum_{j=0}^{T-1}  (\hat{\sigma}^2 + \frac{\hat{\zeta}^2}{p^{(j)}})\pi_j.
    \end{align*}
    We average  \eqref{eq:recursion for r} over $t$ followed by plugging in the inequality above to obtain
    \begin{align*}
        \frac{1}{T}&\sum_{t=0}^{T-1} f_t \le \frac{1}{T} \sum_{t=0}^{T-1} 
        \frac{r_{t} - r_{t+1}}{\eta} + \frac{\eta \hat{\sigma}^2}{m} +\frac{3L}{T} \sum_{t=0}^{T-1} \Xi_t \\
        & \le \frac{r_0}{T\eta } + \frac{\eta \hat{\sigma}^2}{m} + \frac{3 L}{T} \cdot 
        \\
         &\left(   
             (1+\pi_0) \Xi_0 +72\eta^2 L \sum_{j=0}^{T-2} \frac{f_j\pi_j}{p^{(j)}} + 18\eta^2 \sum_{j=0}^{T-2}  (\hat{\sigma}^2 + \frac{\hat{\zeta}^2}{p^{(j)}})\pi_j
        \right) \\
        &\le  \frac{r_0}{T\eta } + \frac{\eta \hat{\sigma}^2}{m} + \frac{3 L (1+\pi_0) \Xi_0}{T} + \frac{216 L^2\eta^2 \Upsilon^2}{T} \sum_{j=0}^{T-2} f_j \\
        &+54L\eta^2  \big(\hat{\sigma}^2 \Pi_1(T) 
        + \hat{\zeta}^2 \Pi_2(T)\big).
    \end{align*}
    We choose the learning rate to be
    \[
    \eta:=\min\left\{ 
            (\frac{mr_0}{\hat{\sigma}^2 T})^{\frac{1}{2}}, 
            \min_t \frac{p^{(t)}}{900L},
            \big(\frac{r_0}{TL(\Pi_1(T) \hat{\sigma}^2 + \hat{\zeta}^2  \Pi_2(T))}\big)^{\frac{1}{3}}
        \right\}.
    \]
    It is straightforward to see that with the current choice $\eta < \frac{1}{12L}$. Thus, $\eta$  satisfies  conditions of both Lemma~\ref{lem : 1 step descent contraction convex} and \ref{lem: 2 step consensus bound convex}.
    Also, observe that by Lemma~\ref{fact: pi} $\pi_t \le \frac{2}{\min_j p^{(j)}}$ for all $t\ge 0$.
    Let $j^*$ denote the index that maximizes $\pi_j/p^{(j)}$.
    It follows that 
    \[
    \max_j \frac{2}{(p^{(j)})^2} \ge \frac{1}{p^{(j^*)}}\cdot  \frac{2}{\min_j p^{(j)}} \ge \frac{\pi_{j^*}}{p^{(j^*)}},
    \]
    so we have $\eta^2 \le (450L^2\Upsilon^2)^{-1}$
    since $\eta < \min_t \frac{p^{(t)}}{900L}$.

    Now we evaluate our upper bound for this choice of $\eta$.
    As $\eta^2 \le (450 L^2 \Upsilon^2)^{-1}$, we always get 
    \[
    \frac{216 L^2\eta^2 \Upsilon^2}{T} \sum_{j=0}^{T-2} f_j \le \frac{1}{2T} \sum_{j=0}^{T-1} f_j.
    \]
    Also, since $\eta\le  \big(\frac{r_0}{TL(\Pi_1(T) \hat{\sigma}^2 + \hat{\zeta}^2  \Pi_2(T))}\big)^{\frac{1}{3}}$, and $\eta^2\le \frac{mr_0}{\hat{\sigma}^2 T}$, we have 
    \[
    L\eta^2  \big(\hat{\sigma}^2 \Pi_1(T) 
        + \hat{\zeta}^2 \Pi_2(T)\big)
        \le \frac{r_0^{2/3}}{T^{2/3}}
            \big( 
                L (\Pi_1(T) \hat{\sigma}^2 + \Pi_2(T) \hat{\zeta}^2)
            \big)^{1/3},
    \]
    \[
    \text{ and } \quad
    \frac{\eta \hat{\sigma}^2 }{m} \le \sqrt{\frac{\hat{\sigma}^2 r_0}{mT}}.
    \]
    Finally, when $\eta = (\frac{mr_0}{\hat{\sigma}^2 T})^{\frac{1}{2}}$, we have 
    \[
    \frac{r_0}{T\eta } = \sqrt{\frac{\hat{\sigma}^2 r_0}{mT}}.
    \]
    When $\eta =  \min_t \frac{p^{(t)}}{900L}$, we obtain
        \[
    \frac{r_0}{T\eta } = \frac{900 r_0 L}{T \min_t p^{(t)}},
    \]
    and when $\eta=\big(\frac{r_0}{TL(\Pi_1(T) \hat{\sigma}^2 + \hat{\zeta}^2  \Pi_2(T))}\big)^{\frac{1}{3}}$ we have
    \[
    \frac{r_0}{T\eta }\le  \frac{r_0^{2/3}}{T^{2/3}}
            \big( 
                L (\Pi_1(T) \hat{\sigma}^2 + \Pi_2(T) \hat{\zeta}^2)
            \big)^{1/3}.
    \]
    Therefore, 
    \begin{align*}
        \frac{1}{T} \sum_{t=0}^{T-1} f_t &\le
           4 \sqrt{\frac{\hat{\sigma}^2 r_0}{mT} } 
            + \frac{6 L(1+\pi_0) \Xi_0}{T} 
            + \frac{1800 r_0 L}{ T\min_t p^{(t)}}\\
        &+
            \frac{220 r_0^{2/3}}{T^{2/3}}
            \big( 
                L (\Pi_1(T) \hat{\sigma}^2 + \Pi_2(T) \hat{\zeta}^2)
            \big)^{1/3},
    \end{align*}
    and the right-hand side can be upper bounded by $\epsilon$ 
    when $T \ge T_4(\Pi_1(T), \Pi_2(T),  \pi_0, p_{\min},  \epsilon,  \overline{\bm{x}}^{(0)})$.
\end{proof}

\subsection{Deferred Proofs}\label{sec: deferred proof}
We now present proofs for Lemma~\ref{fact: pi}, Lemma~\ref{lem:bound under budget D for K=1}, Lemma~\ref{lem:T((p1,tau1),p2)}, and Theorem~\ref{thm: guarantee heterogeneous}.
\begin{proof}[Proof of Lemma~\ref{fact: pi}]
By the definition of \(\pi_j\) and our assumption, we have:
    \begin{align*}
        \pi_j \le \sum_{i>j} \prod_{t=j+1}^{i-1} (1-\frac{\delta}{2})
        \le \sum_{t=0}^\infty (1-\frac{\delta}{2})^t 
        =\frac{1}{1-(1-\frac{\delta}{2})} = \frac{2}{\delta}.
    \end{align*}
We have shown Item 1, and Item 2 follows from the definitions of \(\Pi_1(T)\) and \(\Pi_2(T)\).
For Item 3, 
in the special case of $p^{(t)}\equiv p$, we have $$\pi_0=\pi_1=\pi_2=\dots= \sum_{t=0}^{\infty} \prod_{j=1}^t (1-\frac{p^{(j)}}{2}).$$ Thus,
\begin{equation}
\Pi_1(T) = \pi_0 =    \sum_{t=0}^{\infty} \prod_{j=1}^t (1-\frac{p^{(j)}}{2}) 
 = \sum_{t=0}^{\infty} (1- \frac{p}{2})^t =\frac{2}{p}, 
\end{equation}
and 
\begin{equation}
    \Pi_2(T) = \frac{\pi_0}{p} = \frac{2}{p^2},
\end{equation}
for any $T\ge 1$.
\end{proof}
\begin{proof}[Proof of Lemma~\ref{lem:bound under budget D for K=1}]
For any fixed $i\in V$, by Hoeffding's concentration inequality, 
we have for any $y>0$,
\[
\Pr \left[
\sum_{t=1}^T c_i(\bm{W}^{(t)}) \ge TD+ y \right] \le   
\exp{\left( 
    -\frac{2y^2}{TD^2}
\right)}.
\]
It follows a union bound that 
\[
\Pr \left[
    \max_i \sum_{t=1}^T c_i(\bm{W}^{(t)}) 
    \ge TD+ y 
\right] \le m \cdot \exp{\left( 
    -\frac{2y^2}{TD^2}
\right)},
\]
and by integration over $y$, we obtain that
\begin{equation}
    \E\left[  \max_i \sum_{t=1}^T c_i(\bm{W}^{(t)}) \right] \le D\cdot \left(T
    + m\sqrt{\frac{T \pi}{8}} \right) 
    =: q(T,D), \nonumber
\end{equation}
which complete the proof.
\end{proof}


\begin{proof}[Proof of Lemma~\ref{lem:T((p1,tau1),p2)}]
If $j\ge \tau_1$, then Lemma~\ref{fact: pi} implies that 
$\pi_j = \frac{2}{p_2}$ since $(p^{(t+\tau_1)})_{t\ge 0}$ can be treated as a constant sequence where $p^{t+\tau_1} \equiv p_2$. 
If $j<\tau_1$ then
\begin{align*}
    \pi_j &=  \sum_{i>j}^{\infty} \prod_{t=j+1}^{i-1} (1-\frac{p^{(t)}}{2}) \\ 
     &=  \sum_{i>j}^{\tau_1 } \prod_{t=j+1}^{i-1} (1-\frac{p_1}{2}) 
     + (1-\frac{p_1}{2})^{\tau_1 - j} \sum_{i>\tau_1}^{\infty} \prod_{t=\tau_1+1}^{i-1} (1-\frac{p_2}{2}) \\
     &= \sum_{t=0}^{\tau_1 -j} (1-\frac{p_1}{2})^t 
     + (1-\frac{p_1}{2})^{\tau_1-j} \cdot \pi_{\tau_1}\\
     &= \frac{2[1-(1-\frac{p_1}{2})^{\tau_1-j}]}{p_1} +
     \frac{2(1-\frac{p_1}{2})^{\tau_1-j}}{p_2} \\
     &= \frac{2}{p_1} - (1-\frac{p_1}{2})^{\tau_1-j}(\frac{2}{p_1} - \frac{2}{p_2}). 
\end{align*}
Plugging the above into the definitions \eqref{eq:Pi_1} yields the desired results. 
\end{proof}

\begin{proof}[Proof of Lemma~\ref{lem|:feasibility}]
We verify that the mixing matrix $\bm{W}$ satisfies all the constraints in \eqref{eq: broadcast K=1 obj clique}-\eqref{eq: edge constraint w.p. 1}:
\begin{itemize}
    \item Since we generate $\bm{W}$ from a randomized procedure, $\sum_{\bm{W}} \Pr[\bm{W}] =1$ is clearly satisfied.
    \item Symmetry of $\bm{W}$ also follows from the construction.
    \item The constraints on row sum and column sum also follows from the construction. 
    \item Since off-diagonal entries are only assigned non-zero weights in Line 14, the construction  respects the topology constraint and thus $\bm{W}$ satisfies \eqref{eq: edge constraint w.p. 1}. 
    \item For \eqref{eq: broadcast K=1 constraint clique hetero}, let $i\in V$ be any node such that $c_i^a+c_i^b>D$ (otherwise \eqref{eq: broadcast K=1 constraint clique hetero} holds trivially). Observe that due to the aforementioned reason,
\[
\Pr[\exists j\neq i: \bm{W}[i,j]\neq 0]\le \Pr[i \in U],
\] and the latter probability is set to be at most $\frac{D-c_i^a}{c_i^b}$ in the step 1. Hence, 
\begin{align*}
    c_i^a + 
    \Pr[\exists j\neq i: \bm{W}[i,j]\neq 0] c_i^b 
    \le c_i^a + \frac{(D-c_i^a)}{ c_i^b}c_i^b = D.
\end{align*}
\end{itemize}
Therefore, $\bm{W}$ is a feasible solution to \eqref{eq:budgeted random W design - clique}. 
\end{proof}

\begin{proof}[Proof of Theorem~\ref{thm: guarantee heterogeneous}]
    Notice if the base topology is a clique, then $V_i \equiv V$ for every $i\in V$. 
    Hence, if $i\in U$, Algorithm~\ref{alg:randomized-W for broadcast} assigns $\bm{W}[i,j]=1/|U|$ for all $j\in U$ in Lines 14-15. 
    For any $i,j\in V$ such that $i\neq j$, we have
\begin{align*}
    &\E[(\bm{W}^\top \bm{W})[i,j]] = \E\left[ \sum_{k=1}^m \bm{W}[i,k] \bm{W}[k,j] \mathbbm{1}[i,j,k\in U] \right] \\
    & = \omega_i \omega_j \E\left[ 
        \sum_{k=1}^m \bm{W}[i,k] \bm{W}[k,j] \mathbbm{1}[k\in U]
        \mid  i,j\in U
     \right] \\
     & =  \omega_i \omega_j  \E\left[ 
        |U|^{-2}
        \sum_{k=1}^m \mathbbm{1}[k\in U]
        \mid  i,j\in U
     \right] \\
    &=  \omega_i \omega_j \E\left[  
        \frac{1}{|U|}
        \mid  i,j\in U
     \right] \\ 
     &\sim_{m}  \omega_i \omega_j \E\left[  
        \frac{1}{|U|} \mid U\neq \emptyset
     \right] = \omega_i \omega_j m^\perp.
\end{align*}

Using the observation that $m^\perp$ tends to $0$ as $m$ grows, we compute asymptotically the diagonal entries as
\begin{align*}
  &\E[(\bm{W}^\top \bm{W})[i,i]] \\ &= \big( 1- \omega_i \big) 
  + \omega_i\cdot 
  \E\left[|U|^{-2}\sum_{k=1}^m \mathbbm{1}[k\in U] \mid i\in U \right]
  \\
  &\sim_{m}  1- \omega_i + \omega_i m^\perp 
  \sim_{m}1- \omega_i.
\end{align*}
By combining the steps above, we obtain that 
\begin{align*}
   & \rho(\bm{W} ) = \|\E[\bm{W}^\top\bm{W}]-\bm{J}\| \\
    &\sim_{m} \left\| m^\perp \bm{\omega}\bm{\omega}^\top + \diag(\bm{1}-\bm{\omega} ) - \frac{1}{m} \bm{1}\bm{1}^\top \right\|, 
\end{align*}
as desired.
\end{proof}

\begin{proof}[Proof of Theorem~\ref{thm:rho bound on clique - unicast}]
We will prove that a worse mixing matrix $\bm{W}_H$ based on the union of layered Ramanujan graphs $H=\bigcup_{r=1}^R H^{(r)}$ and a uniform weight of $1/d_R$ for all non-zero off-diagonal entries already satisfies $\rho(\bm{W}_H)
\le
\left(
1-\frac{\dhat_1-2\sqrt{\dhat_1-1}}{d_R}
\right)^2$. Since the optimal weight setting in \eqref{eq:sdp-weight-design} and the optimal statistical mixing in \eqref{eq:budgeted probabilities design} can only reduce $\rho$, this will prove \eqref{eq:rho uppper bound - unicast}.   

Since the final graph \(H\) contains \(\h1\) as a subgraph, its Laplacian satisfies
\[
L_H=L_{\h1}+L_{H\setminus \h1}.
\]
We then have
\[
L_H\succeq L_{\h1}.
\]
Equivalently, for every vector \(\bm x\),
\[
\bm x^\top L_H\bm x
\ge
\bm x^\top L_{\h1}\bm x.
\]
Therefore, by the Rayleigh quotient characterization of the second-smallest Laplacian eigenvalue,
\[
\lambda_2(L_H)
=
\min_{\substack{\bm x\perp \bm 1\\ \bm x\neq 0}}
\frac{\bm x^\top L_H\bm x}{\bm x^\top\bm x}
\ge
\min_{\substack{\bm x\perp \bm 1\\ \bm x\neq 0}}
\frac{\bm x^\top L_{\h1}\bm x}{\bm x^\top\bm x}
=
\lambda_2(L_{\h1}).
\]
Since the first layer \(H^{(1)}\) is a  connected \(\dhat_1\)-regular Ramanujan graph on all nodes, we have
\[
\lambda_2(L_{\h1})\ge \dhat_1-2\sqrt{\dhat_1-1}.
\]
Thus,
\[
\lambda_2(L_H)\ge \dhat_1-2\sqrt{\dhat_1-1}.
\]
We next bound the largest Laplacian eigenvalue. Since
\[
H=\bigcup_{r=1}^R H^{(r)},
\]
we have
\[
L_H \preceq \sum_{r=1}^R L_{H^{(r)}},
\]
Therefore,
\[
\lambda_m(L_H)
\le
\lambda_m\!\left(\sum_{r=1}^R L_{H^{(r)}}\right).
\]
Applying Weyl's inequality for Hermitian matrices to the Laplacian layers gives
\[
\lambda_m\!\left(\sum_{r=1}^R L_{H^{(r)}}\right)
\le
\sum_{r=1}^R \lambda_m(L_{H^{(r)}}).
\]
Hence,
\[
\lambda_m(L_H)
\le
\sum_{r=1}^R \lambda_m(L_{H^{(r)}}).
\]

Since each \(H^{(r)}\) is a \(\dhat_r\)-regular Ramanujan graph, we have
\[
\lambda_m(L_{H^{(r)}})
\le
\dhat_r+2\sqrt{\dhat_r-1}.
\]
Therefore,
\[
\lambda_m(L_H)
\le
\sum_{r=1}^R
\left(\dhat_r+2\sqrt{\dhat_r-1}\right).
\]
Because
\[
\sum_{r=1}^R \dhat_r \leq d_R,
\]
we obtain
\[
\lambda_m(L_H)
\le
d_R+2\sum_{r=1}^R\sqrt{\dhat_r-1}.
\]
If we assign \(
\bm W_H=\bm I-\frac{1}{d_R} L_H,
\)

\[
\rho(\bm W_H)
=
\max\left\{
\left(1-\frac{\lambda_2(L_H)}{d_R}\right)^2,\,
\left(1-\frac{\lambda_m(L_H)}{d_R}\right)^2
\right\}.
\]
First, using the bounds above \(\lambda_2(L_H)\), we have
\[\left(1-\frac{\lambda_2(L_H)}{d_R}\right)^2
\le
\left(
1-\frac{\dhat_1-2\sqrt{\dhat_1-1}}{d_R}
\right)^2,
\]
For \(\lambda_m(L_H)\), if $\lambda_m(L_H)\le d_R$, then all eigenvalues of \(\bm W_H\) are nonnegative
\[\rho(\bm W_H)=
\left(1-\frac{\lambda_2(L_H)}{d_R}\right)^2
\le
\left(
1-\frac{\dhat_1-2\sqrt{\dhat_1-1}}{d_R}
\right)^2,
\]
and if $\lambda_m(L_H)\ge d_R$, we have
\[
\left(1-\frac{\lambda_m(L_H)}{d_R}\right)^2
\le
\left(
\frac{2\sum_{r=1}^R\sqrt{\dhat_r-1}}{d_R}
\right)^2. 
\]
Moreover, since
\[
2\sqrt{\dhat_r-1}\le \dhat_r,\qquad r\ge 2,
\]
We have
\[
2\sum_{r=1}^R\sqrt{\dhat_r-1}
\leq
2\sqrt{\dhat_1-1}+\sum_{r=2}^R\dhat_r
\le
2\sqrt{\dhat_1-1}+d_R-\dhat_1.
\]
Hence,
\[
\left(
\frac{2\sum_{r=1}^R\sqrt{\dhat_r-1}}{d_R}
\right)^2
\le
\left(
1-\frac{\dhat_1-2\sqrt{\dhat_1-1}}{d_R}
\right)^2.
\]
Combine both bounds,
\[
\rho(\bm W_H)
\le
\left(
1-\frac{\dhat_1-2\sqrt{\dhat_1-1}}{d_R}
\right)^2.
\]

\end{proof}

\begin{proof}[Proof of Theorem~\ref{lem:finiteness}]
Since all candidate graphs are iteratively generated by calling the same random graph oracle, we refer to each call of the oracle an ``iteration''.  
    Let $\tau'$ be the first iteration $k$ when for every $e\in E$ there exists an iteration $j\le k$  such that $e\in E_j$.
    We first prove that $\tau \le \tau'$ by showing that $\rho_{\tau'}<1$.
    For any $\epsilon\in (0,1)$, consider a probability vector $$\bm{p}':=(p_0=1-\epsilon, p_1=\epsilon/\tau', \dots, p_{\tau'}=\epsilon/\tau').$$
    By taking $\epsilon$ sufficiently small, the constraint \eqref{eq: unicast K=1 constraint hetero} can be satisfied by $\bm{p}'$. 
    As $\rho_{\tau'}$ is obtained by minimizing the problem defined in \eqref{eq:budgeted probabilities design}, we have
    \[
    \rho_{\tau'} \le 
        \left\|(1-\epsilon)I+ \frac{\epsilon}{\tau'}
        \sum_{s=1}^{\tau'} \bm{W}_s^\top \bm{W}_s -J
        \right\| =: \rho^*. 
    \]
     Denote $\frac{1}{\tau'} \sum_{s=1}^{\tau'} \bm{W}_s $ by $\overline{\bm{W}}$.  
     By triangle inequality, we get that
    \[
    \rho^* \le (1-\epsilon) + \epsilon  \left
        \|\frac{1}{\tau'}
        \sum_{s=1}^{\tau'} (\bm{W}_s^\top \bm{W}_s -J)\right\| 
    \]
     According to the weight assignment in Lines 6-9, 
    every $\bm{W}_s$ is symmetric and double-stochastic. Thus $\overline{\bm{W}}$ is also symmetric and double-stochastic.  
     For any double-stochastic matrix $\bm{W}$ we have $\bm{W}^\top J=J$ and $\| \bm{W} \|\le 1$, and it follows that
    \begin{align*}
        &\left
        \|\frac{1}{\tau'}
        \sum_{s=1}^{\tau'} (\bm{W}_s^\top \bm{W}_s -J)\right\| 
        =  \left \|\frac{1}{\tau'}
        \sum_{s=1}^{\tau'} (\bm{W}_s^\top \bm{W}_s -\bm{W}_s^\top J)\right\| \\
        &= \left \|\frac{1}{\tau'}
        \sum_{s=1}^{\tau'} \bm{W}_s^\top (\bm{W}_s - J)\right\| 
        \le \max_s \left \| \bm{W}_s^\top (\overline{\bm{W}} -J) \right\| \\
        &\le  \max_s \left \| \bm{W}_s^\top  \right\|
        \cdot \|\overline{\bm{W}} -J \|
        \le  \|\overline{\bm{W}} -J \|.
    \end{align*}  
    Moreover, by the choice of $\tau'$, $\overline{\bm{W}}$ is  an irreducible matrix, and thus by Perron-Frobenius theorem, 
    $\overline{\bm{W}}$ has the following eigenvalues:
    \[
    1 = \lambda_1(\overline{\bm{W}})>\lambda_2(\overline{\bm{W}})>\dots > \lambda_m(\overline{\bm{W}}) > -1,
    \]
    and the eigenvector corresponding to the leading eigenvalue is the all-one vector $\bm{1}$.
    On the other hand, $J$ has only 1 non-zero eigenvalue whose eigenvector is also  the all-one vector $\bm{1}$.
    Hence, 
    \[
    \|\overline{\bm{W}} -J \| = \max(|\lambda_2(\overline{\bm{W}})|, |\lambda_m(\overline{\bm{W}})|)<1,
    \]
    and this implies that 
    \[
    \rho_{\tau'} \le \rho^* \le (1-\epsilon) + \epsilon \|\overline{\bm{W}} -J \| < 1.
    \]
    Next we prove that there exists $C>0$ such that \[\Pr[\tau' > Cm]\le e^{-\Omega(m)},\]
    from which the theorem follows.
    By assumption, there is a constant $c>0$ such that for each $e\in E$ and every iteration $k$, $e\notin E_k$ occurs with probability at most $1-c$.
    Then $e\notin E_j$ for any of $j=1,\dots, t$ iterations occurs with probability at most $(1-c)^{t}$.
    By a union bound over $e\in E$, 
    \[
    \Pr[\tau' > t] \le m^2 \cdot (1-c)^t \le m^2 \cdot \exp{(-ct)}.
    \]
    By letting $t=Cm/c$ with $C$ sufficiently large, 
    the right-hand side is at most $e^{-\Omega(m)}$.
\end{proof}




\if\thisismainpaper0

\end{document}